\documentclass[12pt]{IEEEtran}
\onecolumn

\usepackage{subfigure}
\usepackage{wrapfig}
\usepackage[utf8]{inputenc} 
\usepackage[T1]{fontenc}    
\usepackage{url}            
\usepackage{booktabs}       
\usepackage{amsfonts}       
\usepackage{nicefrac}       
\usepackage{microtype}      
\usepackage{xcolor}         

\usepackage{tabularx} 

\usepackage{algorithm,algorithmic}

\usepackage{epsfig,amsmath,amssymb,amstext,amsthm,mathrsfs}
\usepackage{latexsym,graphics,epsf,epsfig,psfrag}
\usepackage{dsfont,color,epstopdf,fixmath}
\usepackage{enumitem}

\usepackage{hhline}

\usepackage{threeparttable}
\usepackage{makecell}

\usepackage{amsfonts}       
\usepackage{microtype}      
\usepackage{subfigure}
\usepackage{mathtools}
\usepackage{amssymb}
\usepackage{amsthm}

\usepackage{cleveref}

\usepackage{multirow}

\usepackage{multirow}
\newcommand{\mcs}{\mathcal{S}}
\newcommand{\mca}{\mathcal{A}}
\newcommand{\nn}{\nonumber}
\newcommand{\mE}{\mathbb{E}}

\newtheorem{assumption}{\textbf{Assumption}}

\newtheorem{corollary}{\textbf{Corollary}}
\newtheorem{Lemma}{\textbf{Lemma}}
\newtheorem{theorem}{\textbf{Theorem}}

\newtheorem{remark}{\textbf{Remark}}

\usepackage[acronym]{glossaries} 

\allowdisplaybreaks[4]

\title{Finite-Time Error Bounds for Greedy-GQ}

\author{Yue Wang \quad Yi Zhou \quad Shaofeng Zou
	\thanks{ Yue Wang is with the Department of Electrical and Computer Engineering, University of Central Florida (email: yue.wang@ucf.edu); Yi Zhou is with the Department of Electrical and Computer Engineering, University of Utah (yi.zhou@utah.edu); Shaofeng Zou is with the Department of Electrical Engineering, University at Buffalo ( szou3@buffalo.edu).
}
}

\begin{document}

\maketitle
\begin{abstract}
    Greedy-GQ  with linear function approximation, originally proposed in \cite{maei2010toward}, is a value-based off-policy algorithm for optimal control in reinforcement learning, and it has a non-linear two timescale structure with non-convex objective function. This paper develops its tightest finite-time error bounds. We show that the Greedy-GQ algorithm converges as fast as $\mathcal{O}({1}/{\sqrt{T}})$ under the i.i.d.\ setting and $\mathcal{O}({\log T}/{\sqrt{T}})$ under the Markovian setting. We further design variant of the vanilla Greedy-GQ algorithm using the nested-loop approach, and show that its sample complexity is $\mathcal{O}({\log(1/\epsilon)\epsilon^{-2}})$, which matches with the one of the vanilla Greedy-GQ. Our finite-time error bounds match with the one of the stochastic gradient descent algorithm for general smooth non-convex optimization problems, despite of its additonal challenge in the two time-scale updates. Our finite-sample analysis provides theoretical guidance on choosing step-sizes for faster convergence in practice, and suggests the trade-off between the convergence rate and the quality of the obtained policy. Our techniques provide a general approach for finite-sample analysis of non-convex two timescale value-based reinforcement learning algorithms.
\end{abstract}

\section{Introduction}
Recent success of reinforcement learning (RL) in benchmark tasks suggests a potential revolutionary advance in practical applications and has dramatically boosted the interest in RL. 
However, common algorithms are highly data-inefficient, making impressive results only on simulated systems, where an infinite amount of data can be simulated. 
Such data-inefficiency limits the application of RL algorithms in many practical applications without a large amount of data. Even though, theoretical understanding of various RL algorithms' sample complexity is still very limited,  resulting in an over-reliance on empirical experiments without sufficient principles to guide future development. This motivates the study of finite-time error bounds for RL algorithms in this paper.

In the problem of RL \cite{sutton2018reinforcement}, an agent interacts with a stochastic environment by taking a sequence of actions according to a policy, and aims to maximize its expected accumulated reward.  The environment is modeled as a Markov decision process (MDP), which consists of a state space, an action space, and an action dependent transition kernel. When the state and action spaces are small, the tabular approach is usually used, which stores the action-values using a table. In this case, the optimal policy can be found \cite{sutton2018reinforcement,watkins1992q}, e.g., using Q-learning. However, in many applications, the state and action spaces can be very large or even continuous, and thus the tabular approach is not applicable anymore. Then, the approach of function approximation is usually used, where a parameterized family of functions, e.g., neural network, is used to approximate the action-value function and/or the policy. The problem is to estimate the parameter of the function, e.g., weights of a neural network, which has a much lower dimension than the total number of state-action pairs.

Despite the broad application and success of RL algorithms with function approximation in practice, 
RL methods, e.g., temporal difference (TD), Q-learning and SARSA,  when combined with function approximation approaches may not converge under off-policy training \cite{baird1995residual,gordon1996chattering}.
To address the non-convergence issue in off-policy training, a class of gradient temporal difference (GTD) learning algorithms were developed in \cite{maei2010toward,maei2011gradient,sutton2009fast,Sutton2009b}, including GTD, GTD2, TD with correction term (TDC), and Greedy Gradient Q-learning (Greedy-GQ). The basic idea is to construct squared objective functions, e.g., mean squared projected Bellman error, and then to perform gradient descent. To address the double sampling problem in gradient estimation, a weight doubling trick was proposed in \cite{sutton2009fast}, which leads to a two timescale update rule and biased gradient estimate. 
One great advantage of this class of algorithms is that they can be implemented in an online and incremental fashion, which is memory and computationally efficient.

The asymptotic convergence of these two timescale algorithms has been well studied under both i.i.d.\ and non-i.i.d.\ settings, e.g., \cite{sutton2009fast,Sutton2009b,maei2010toward,yu2017convergence,borkar2009stochastic,borkar2018concentration,karmakar2018two}. The finite-time error bounds of these algorithms are of great practical interest for algorithmic parameter tuning and design of new sample-efficient algorithms, which, however, remain unsolved until very recently \cite{dalal2018finite,wang2017finite,liu2015finite,gupta2019finite,xu2019two}.
Existing finite-sample analyses are only for the GTD, GTD2 and TDC algorithms, which are designed for policy evaluation. The finite-sample analysis for the Greedy-GQ algorithm, which is to directly learn an optimal control policy, is still not understood and will be the focus of this paper.

In this paper, we develop the finite-time error bound for the Greedy-GQ algorithm, and prove that it converges as fast as $\mathcal{O}({1}/{\sqrt{T}})$ under the i.i.d.\ setting and $\mathcal{O}\left({\log T}/{\sqrt{T}}\right)$ under the Markovian setting. We further propose one variant of the vanilla Greedy-GQ algorithm using the nested-loop approach that is commonly used in optimization and practical RL applications and show that its overall sample complexity is $\mathcal{O}(\epsilon^{-2})$ which is the same as the vanilla Greedy-GQ (up to a $\mathcal{O}(\log(1/\epsilon))$ factor).

\subsection{Main Challenges and Contributions}
The major challenge in the finite-time analysis of two timescale algorithms lies in developing a tight bound on the the tracking error, which measures how the fast timescale tracks its ideal limit (see \eqref{eq:bias_decompose}). Specifically, we develop a novel and more refined technique that bound the tracking error in terms of the gradient norm of the objective function. More importantly, our analysis is developed for the vanilla Greedy-GQ algorithm. The vanilla Greedy-GQ updates the parameters in an online and incremental fashion for each new sample, and is easier to implement and is more practical. Compared to the analysis for variants of the Greedy-GQ algorithm using the mini-batch approach (\cite{xu2020sample}), variance reduced approach (\cite{ma2020variance}), and the nested loop approach (designed in this paper),  the analysis of the vanilla Greedy-GQ algorithm is more challenging since the tracking error and the variance cannot be simply controlled by choosing a large batch size or a large number of inner loop iterations, which is typically $\mathcal O(1/\epsilon)$. Also note that by setting the batch size equal to 1, the finite-time error bounds in \cite{xu2020sample,ma2020variance} reduce to a non-zero constant, which does not vanish with iterations.

Different from existing studies on two timescale stochastic approximation algorithms, e.g., \cite{konda2004convergence,dalal2020tale,kaledin2020finite,gupta2019finite}, where the objective function is convex, and the updates are linear in the parameters, the objective function of the Greedy-GQ algorithm is non-convex, and is not always differentiable, and the updates of Greedy-GQ are non-linear. Therefore, convergence to global optimum cannot be guaranteed in general.  In this paper, we focus on the convergence to stationary points, i.e., the convergence of gradient norm to zero. The two timescale update nature of the algorithm introduces bias in the gradient estimate, which makes the analysis  significantly different from standard non-convex optimization analysis \cite{ghadimi2013stochastic} (one timescale with unbiased stochastic gradient). Since the updates are non-linear in the parameters, the approach in \cite{konda2004convergence,kaledin2020finite} that decouples the fast and slow timescale updates via a linear mapping is not directly applicable here. Furthermore, the global convergence of non-linear two timescale stochastic approximation established in \cite{mokkadem2006convergence,doan2021nonlinear} relies on the assumption that the algorithm converges to the global optimum, which however does not necessarily hold for the Greedy-GQ algorithm in this paper.

In this paper, we develop finite-time error bounds for both the i.i.d.\ setting and the Markovian setting. We develop novel analysis of the bias and the tracking error addressing the challenges from the Markovian noise and the two timescale structure. Our analysis implies the vanilla Greedy-GQ algorithm converges as fast as the existing invariant algorithms, without using a sample batch of size $\mathcal{O}(\epsilon^{-1})$. Our approach further provides a general framework for analyzing RL algorithms with non-convex objective functions and two timescale updates.

We further propose a variant of the vanilla Greedy-GQ algorithm using the nested-loop approach, which is commonly used in practice to reduce the variance. We prove that the variance and the tracking error can be arbitrarily small by choosing a proper number of inner loop iterations. We derive the finite-time error bound for the nested-loop Greedy-GQ algorithm, and characterize its sample complexity $\mathcal{O}(\log(1/\epsilon)\epsilon^{-2})$, which is the same as the vanilla Greedy-GQ algorithm (up to a factor of $\mathcal O(\log(1/\epsilon))$. We note that another variant of the vanilla Greedy-GQ algorithm using the mini-batch approach was studied in \cite{xu2020sample}, and its sample complexity is shown to be $\mathcal{O}(\epsilon^{-2})$. The two variants share the same order of sample complexity as the vanilla Greedy-GQ algorithm, and are much easier to analyze than the vanilla Greedy-GQ algorithm since the variance and bias in the gradient estimate can be made arbitrarily small by choosing a large batch size or a large number of inner loops, however at the price of updating only after every big batch of data of size $\mathcal O(1/\epsilon)$ and losing the simplicity of implementation in practice. Our analysis further illustrates that the sample batch based approaches can reduce the variance and stabilize the algorithm, but cannot not improve the convergence rate or the sample complexity.
 
\subsection{Related Work}
In this section, we discuss recent advances on the finite-time error bounds for various value-based RL algorithms with function approximation.
There are also recent studies on actor-critic algorithms and policy gradient algorithms, which are not the focus of this paper, and thus are not discussed here.

\textbf{TD methods.} For policy evaluation problems, the finite-time error bound for TD learning was developed in \cite{srikant2019,lakshminarayanan2018linear,bhandari2018finite,Dalal2018a,sun2020finite}. For optimal control problems, the finite-time error bounds for  SARSA with linear function approximation were developed in \cite{zou2019finite}. The finite-time error bounds for TD and Q-learning with neural function approximation were studied in \cite{cai2019neural,xu2020finite}. These methods are only applicable to on-policy setting or require additional conditions to guarantee convergence. These algorithms are one timescale,  and thus their analyses are fundamentally different from ours.

\textbf{Linear two timescale methods.} The asymptotic convergence rate of general linear two timescale SA algorithms was first developed in \cite{konda2004convergence}. Recently, finite-time error bounds for various linear two timescale RL algorithms, e.g., GTD(0), GTD, and TDC, were studied, e.g., \cite{xu2019two,dalal2020tale,gupta2019finite,kaledin2020finite,ma2020variance}. In these studies, the parameters are updated as linear functions of their previous values. Moreover, the objective function is convex, and thus the convergence to the global optimum can be established. For the Greedy-GQ algorithm, the updates are non-linear, and the objective function is non-convex. Thus, new techniques are needed to develop the finite-time error bounds.

\textbf{Non-linear two timescale methods.}  The asymptotic convergence rate and finite-time error bounds for general nonlinear two-time scale SA were studied in \cite{mokkadem2006convergence,doan2021nonlinear,bhatnagar2009convergent,xu2020sample}. In \cite{mokkadem2006convergence,doan2021nonlinear}, it is assumed that the algorithm converges to the global optimum. However, for Greedy-GQ, the objective function is non-convex, and the convergence to the global optimum cannot be guaranteed. In \cite{bhatnagar2009convergent,xu2020sample}, the GTD method with non-linear function approximation was proposed and investigated, where the asymptotic convergence was established in \cite{bhatnagar2009convergent}, and the sample complexity of the mini-batch variant was established in \cite{xu2020sample}. 
{Among the previous analyses of non-linear two-timescale algorithms, the most relevant works are \cite{ma2021greedy,xu2020sample}, where two variants of the vanilla Greedy-GQ algorithm are introduced: mini-batch Greedy-GQ and variance-reduced Greedy-GQ. Our results differ from theirs in two key aspects: the updating methodology of the algorithm and the technique employed for sample complexity analysis. Specifically, both algorithms generate a batch of samples of size $\mathcal{O}(\epsilon^{-1})$ before updating parameters at each time step to control the tracking error and obtain their results, whereas we only generate a single sample, hugely reducing memory and computation costs. Moreover, we develop a novel technique for bounding the tracking error without using batch size but still obtain a matching sample complexity, except for a negligible $\log\epsilon^{-1}$ term. It is also worth mentioning that our results cannot be obtained by directly setting the sample batch size to $1$ in the outcomes of \cite{ma2021greedy,xu2020sample}, which would result in a constant bound. Comparisons with these two works further highlight our contribution: a novel technique that enables us to attain the tightest bound for an online-updating algorithm without introducing a large-size sample batch.} 


\color{black}
\section{Preliminaries}

\subsection{Markov Decision Process}
A discounted Markov decision process(MDP) consists of a tuple $(\mathcal{S},\mathcal{A},  \mathsf P, r, \gamma)$, where $\mathcal{S}$ and $\mathcal{A}$ are the state and action spaces, $\mathsf P$ is the transition kernel, $r$ is the reward function, and $\gamma\in [0,1]$ is the discount factor. Specifically, at each time $t$, an agent takes an action $a_t\in\mathcal{A}$ at state $s_t\in\mathcal{S}$. The environment then transits to the next state $s_{t+1}$ with probability $\mathsf P(s_{t+1}\vert s_t,a_t)$, and the agent receives reward given by $r(s_t,a_t,s_{t+1})$. 
A stationary policy $\pi$ maps a state $s \in \mathcal{S}$ to a probability distribution $\pi(\cdot\vert s)$ over $\mathcal{A}$, which does not depend on time $t$. For a given policy $\pi$,  we define its value function for any initial state $s\in \mathcal{S}$ as
$V^\pi\left(s\right)=\mathbb{E}\left[\sum_{t=0}^{\infty}\gamma^t   r(S_t,A_t,S_{t+1})\vert S_0=s\right],$
and the state-action value function (i.e., the $Q$-function) for any $(s,a)\in\mathcal{S}\times\mathcal{A}$ as  
$
	Q^\pi(s,a)=\mE_{S'\sim \mathsf P(\cdot\vert s,a)}\left[r(s,a,S')+\gamma V^\pi(S')\right].
$

\subsection{Linear Function Approximation}
In practice, the state and action spaces can be extremely large, which makes the tabular approach intractable due to high memory and computational cost. The approach of function approximation is widely used to address this issue, where the Q-function is approximated using a family of parameterized functions. In this paper, We focus on linear function approximation. Specifically,  
let $\left\{ \phi^{(i)}: \mcs\times\mca\rightarrow \mathbb R,\, i=1,\ldots,N \right\}$ be a set of $N$ fixed base functions, where $N \ll \vert  \mcs\vert  \times\vert  \mca\vert  $. In particular, we approximate the Q-function using a linear combination of $\phi^{(i)}$'s:  
\begin{flalign}
Q_\theta(s,a)=\sum_{i=1}^N \theta(i)\phi^{(i)}_{s,a}=\phi_{s,a}^\top \theta,
\end{flalign}
where $\theta \in \mathbb{R}^N$ is the weight vector. Denote by $\mathcal Q$  the family of linear combinations of $\phi^{(i)}$'s: $\mathcal Q= \left\{Q_{\theta}:\theta\in\mathbb{R}^N \right\}$.
The goal is to find a $Q_\theta\in\mathcal Q$ with a compact representation in $\theta$ to approximate the optimal action-value function $Q^*$.

\subsection{Greedy-GQ Algorithm}
In this subsection,  we review the design of the Greedy-GQ algorithm,  which was originally proposed in \cite{maei2010toward} to solve the  optimal control problem in RL under off-policy training.

Greedy-GQ algorithm is to minimize the mean
squared projected Bellman error (MSPBE):
\begin{flalign}\label{eq:objective}
J(\theta)\triangleq\Vert \mathbf \Pi\mathbf T^{\pi_{\theta}}Q_{\theta}-Q_{\theta}\Vert ^2_{\mu},
\end{flalign}
where $\mu$ is the stationary distribution induced by the behavior policy $\pi_b$, $\pi_\theta$ is a policy derived from $Q_\theta$, e.g., softmax w.r.t. $Q_\theta$, and $\Vert Q(\cdot,\cdot)\Vert _\mu\triangleq\int_{s\in\mcs,a\in\mca}d\mu_{s,a}Q(s,a)$, $\mathbf \Pi$ is a projection operator $\mathbf \Pi \hat Q=\arg\min_{Q\in \mathcal Q}\Vert Q-\hat Q\Vert _\mu$, and $\mathbf T^{\pi_{\theta}}$ is the Bellman operator w.r.t. policy $\pi_\theta$.

Two sampling settings, i.i.d.\ and Markovian, are considered, where under the i.i.d.\ setting, i.i.d.\ tuples $(s_t,a_t,s_t')$ generated from the stationary distribution $\mu$ induced by the behavior policy $\pi_b$ are obtained, and under the Markovian setting, a single sample trajectory  $(s_t,a_t,s_{t+1})$ is obtained by following the behavior policy. In this paper, the proof and  discussion will mainly focus on the Markovian setting. We will also present the results for the i.i.d.\ setting, whose proof will be omitted and can be easily obtained from proof for the Markovian setting.

Let $\delta_{s,a,s'}(\theta)=r({s,a,s'})+\gamma\Bar{V}_{s'}(\theta)-\theta^\top  \phi_{s,a}$ be the TD error, where $\Bar{V}_{s'}(\theta)=\sum_{a'} \pi_{\theta}(a'\vert  s')\theta^\top  \phi_{s',a'}$. Then $J\left (\theta\right ) =\mathbb{E}_\mu\left [\delta_{S,A,S'}\left (\theta\right )\phi_{S,A}\right ]^\top
 C^{-1}\mathbb{E}_\mu\left [\delta_{S,A,S'}\left (\theta\right )\phi_{S,A}\right ]$, where $C=\mathbb{E}_{\mu}[\phi_{S,A}\phi_{S,A}^\top]$. Let  $\hat{\phi}_{s'}(\theta)=\nabla \Bar{V}_{s'}(\theta)$, then the gradient of $\frac{J(\theta)}{2}$ can be computed as follows: 
\begin{align}\label{eq:5}
    -\mathbb{E}_{\mu}[\delta_{S,A,S'}(\theta)\phi_{S,A}]+\gamma\mathbb{E}_{\mu}[\hat{\phi}_{S'}(\theta)\phi_{S,A}^\top  ]\omega^*(\theta),
\end{align}
where $\omega^*(\theta)=C^{-1} \mathbb{E}_{\mu}[\delta_{S,A,S'}({\theta})\phi_{S,A}].$

To solve the double-sampling issue when estimating the term $\mathbb{E}_{\mu}[\hat{\phi}_{S'}(\theta)\phi_{S,A}^\top  ]\omega^*(\theta)$, which involves the product of two expectations, the weight doubling trick proposed in \cite{sutton2009fast} was used to construct the following two-time scale update of Greedy-GQ:
\begin{align}
    &\theta_{t+1}=\theta_t+\alpha(\delta_{t+1}(\theta_t)\phi_t-\gamma(\omega_t^\top  \phi_t)\hat{\phi}_{t+1}(\theta_t)),\label{eq:thetaupdate}\\
    &\omega_{t+1}=\omega_t+\beta(\delta_{t+1}(\theta_t)-\phi_t^\top  \omega_t)\phi_t,\label{eq:omegaupdate}
\end{align}
where we denote $\delta_{s_t,a_t,s_{t+1}}(\theta)$ by $\delta_{t+1}(\theta)$, $\hat{\phi}_{s_{t+1}}(\theta)$ by $\hat{\phi}_{t+1}(\theta)$ and $\phi_{s_t,a_t}$ by $\phi_t$ for simplicity, and $\alpha$ and $\beta$ are the step-sizes. We also refer the readers to \cite{maei2010toward} for more details of the algorithm construction. 


\begin{algorithm}[tb]
\caption{Greedy-GQ \cite{maei2010toward}}
\label{alg:1}
\mbox{\textbf{Initialization}: $T$, $\theta_0$, $\omega_0$, $s_0$, $\phi^{(i)}$, for $i=1,2,...,N$}
\begin{algorithmic}
\STATE {Choose $W\sim \text{Uniform}(0,1,...,T-1)$}
\FOR {$t=0,1,2,...,W-1$}
		\STATE {Choose $a_t$ according to $\pi_b(\cdot\vert  s_t)$}
		\STATE Observe $s_{t+1}$ and $r_{t}$
		\STATE $\Bar{V}_{s_{t+1}}(\theta_{t}) \leftarrow \sum_{a'\in \mathcal{A}} \pi_{\theta_{t}}(a' \vert  s_{t+1})\theta_{t}^\top  \phi_{s_{t+1},a'}$
		\STATE $\delta_{t+1}(\theta_{t})\leftarrow r_{t}+\gamma\Bar{V}_{s_{t+1}}(\theta_{t})-\theta_{t}^\top  \phi_{t} $
		\STATE $\hat{\phi}_{t+1}(\theta_{t})\leftarrow$  $\nabla \Bar{V}_{s_{t+1}}(\theta_{t})$
		\STATE $\theta_{t+1} \leftarrow  \theta_{t}+\alpha(\delta_{t+1}(\theta_{t})\phi_{t}-\gamma(\omega_{t}^\top  \phi_{t})\hat{\phi}_{t+1}(\theta_{t}))$
		\STATE $\omega_{t+1} \leftarrow \omega_{t}+\beta(\delta_{t+1}(\theta_{t})-\phi_{t}^\top  \omega_{t})\phi_{t}$
		\ENDFOR
\end{algorithmic}
\textbf{Output}: $\theta_W$
\end{algorithm}

\section{Finite-Time Error Bound for Greedy-GQ}
In this section, we present our results of tight finite-time error bounds for the vanilla Greedy-GQ algorithm under both the
i.i.d.\ and Markovian settings. 

\subsection{Technical Assumptions}
We adopt the following standard technical assumptions.
\begin{assumption}[Problem solvability]\label{assump:C}
$C$ is non-singular and its smallest eigenvalue is denoted by $\lambda$.
\end{assumption} 
Assumption \ref{assump:C} is a commonly used assumption in the literature, e.g., \cite{srikant2019,xu2019two,bhandari2018finite} to guarantee the problem is solvable.
\begin{assumption}[Bounded feature]\label{assump:boundedfeature}
  $\Vert \phi_{s,a}\Vert _2\leq 1, \forall (s,a)\in\mcs\times\mca$.
\end{assumption}
Assumption \ref{assump:boundedfeature} can be easily guaranteed by normalizing the features.

In this paper, we focus on policies that are smooth. Specifically,  $\pi_{\theta}(a\vert  s)$ and $\nabla \pi_{\theta} (a\vert  s)$ are Lipschitz functions of $\theta$, for any $(s,a)\in \mcs\times\mca$. 
\begin{assumption}[Smooth Policy]\label{assump:policy}
 The policy $\pi_\theta(a\vert  s)$ is $k_1$-Lipschitz and $k_2$-smooth, i.e., for any $(s,a) \in \mcs\times\mca$, $\Vert \nabla \pi_{\theta}(a\vert  s)\Vert \leq k_1, \forall \theta$, and $\Vert \nabla\pi_{\theta_1}(a\vert  s)-\nabla\pi_{\theta_2}(a\vert  s)\Vert  \leq k_2 \Vert  \theta_1-\theta_2\Vert , \forall \theta_1,\theta_2$.
\end{assumption}

To justify the feasibility of Assumption \ref{assump:policy} in practice, in the following, we provide an example of the softmax policy, and show that it is Lipschitz and smooth in $\theta$.

Consider the softmax operator, where for any $(a,s)\in \mca\times\mcs$ and $\theta\in \mathbb R^N$,
\begin{align}\label{eq:softmax}
	\pi_{\theta}(a\vert  s)=\frac{e^{\sigma {\theta}^\top \phi_{s,a}}}{\sum_{a' \in \mathcal{A}}e^{\sigma {\theta}^\top  \phi_{s,a'}}},
\end{align} 
for some $\sigma>0$. 
\begin{Lemma}\label{lemma:softmax_smooth}
	The softmax policy $\pi_{\theta}(a\vert  s)$ is $2\sigma$-Lipschitz and $8\sigma^2$-smooth, i.e., for any $(s,a)\in\mcs\times\mca$, and for any $\theta_1,\theta_2\in\mathbb R^N$, 
	\begin{align}
		\vert  \pi_{\theta_1}(a\vert  s)-\pi_{\theta_2}(a\vert  s)\vert   &\leq 2\sigma \Vert \theta_1-\theta_2 \Vert ,\\
		\Vert \nabla\pi_{\theta_1}(a\vert  s)-\nabla \pi_{\theta_2}(a\vert  s)  \Vert &\leq 8\sigma^2 \Vert \theta_1-\theta_2 \Vert .
	\end{align}
\end{Lemma}

 The following assumption is  needed for the analysis under the Markovian setting, and is a widely adopted assumption for analysis with Markovian samples, e.g., \cite{bhandari2018finite,zou2019finite,xu2019two,srikant2019,cai2019neural}.
\begin{assumption}[Geometric uniform ergodicity]\label{ass:1}
 There exist some constants $m>0$ and $\rho \in (0,1)$ such that 
\begin{align}
    \sup_{s\in\mathcal{S}} d_{TV}(\mathbb{P}(s_t \vert  s_0=s), \mu) \leq m\rho^t ,
\end{align}
for any $t>0$, where $d_{TV}$ is the total-variation distance between the probability measures.
\end{assumption}

\subsection{Finite-time Error Bound and Sample Complexity}
The objective function of Greedy-GQ in \eqref{eq:objective} is non-convex, hence it is not guaranteed to converge to the global optimum. Instead, we consider the convergence to stationary points, namely, we study the rate of the gradient norm converging to zero \cite{ghadimi2013stochastic}. Furthermore, motivated by the randomized stochastic gradient method in \cite{ghadimi2013stochastic}, which is designed to analyze non-convex optimization problems, in this paper, we also consider a randomized version of the Greedy-GQ algorithm in Algorithm \ref{alg:1}. Specifically, let $W$  be an independent random variable with a uniform distribution over $\left\{ 0,1,...,T-1\right\}$. We then run the Greedy-GQ algorithm for $W$ steps. The final output is $\theta_W$.

In the following theorem, we provide the finite-time error bound for $\mathbb{E}[\Vert \nabla  J(\theta_W)\Vert ^2]$.
\begin{theorem}\label{thm:main}
 Consider the following step-sizes: $\beta=\mathcal{O}\left(\frac{1}{T^b}\right)$, and $\alpha=\mathcal{O}\left(\frac{1}{T^a}\right)$, where $\frac{1}{2}\leq a\leq 1$ and $0<b\leq a$. Then, 
(a) under the i.i.d.\ setting, 
\begin{align}
    \mathbb{E}[\Vert \nabla  J(\theta_W)\Vert ^2]\leq      \mathcal{O}\Bigg(\frac{1}{T^{1-a}}+\frac{1}{T^{1-b}}+\frac{1}{T^b}\Bigg);
\end{align}
(b) and under the Markovian setting,
\begin{align}
    \mathbb{E}[\Vert \nabla  J(\theta_W)\Vert ^2]\leq \mathcal{O}\left(\frac{\log T}{T^{1-a}}+\frac{1}{T^{1-b}} +\frac{\log T}{T^b}\right).
\end{align}
\end{theorem}
The proof for the Markovian setting can be found in Appendix \ref{app:A}. The proof for the i.i.d.\ setting is can be obtained by letting $m=0$ in Assumption \ref{ass:1}.
Here we provide the order of the bounds in terms of $T$ for simplicity. The explicit bounds can be found in \eqref{eq:markovianbound} in the appendix. 
Compared to the i.i.d.\ setting, the Markovian setting introduces significant challenges due to
the highly dependent nature of the data. In this case, the bound is essentially scaled by a factor of the
mixing time, $\log T$, relative to the i.i.d.\ case, due to the geometric mixing time in Assumption \ref{ass:1}.

Theorem \ref{thm:main} characterizes the relationship between the convergence rate and the choice of the step-sizes $\alpha$ and $\beta$. We further optimize over the choice of the step-sizes and obtain the following corollary.
\begin{corollary}\label{col:1}
If we choose $a=b=\frac{1}{2}$, then under the i.i.d.\ setting, we have that
$
\mathbb{E}[\Vert \nabla  J(\theta_W)\Vert ^2]=\mathcal{O}\left(\frac{1}{\sqrt{T}}\right);
$
and  under the Markovian setting, we have that
$
\mathbb{E}[\Vert \nabla  J(\theta_W)\Vert ^2]=\mathcal{O}\left(\frac{\log T}{\sqrt{T}}\right).
$
\end{corollary}

Note that  our result matches with the result in \cite{ghadimi2013stochastic} for solving general smooth non-convex optimization problems using stochastic gradient descent. Compared to the analysis in \cite{ghadimi2013stochastic}, our analysis is novel and challenging since the update rule of the Greedy-GQ algorithm is a two timescale one, for which the analysis of the tracking error is challenging, whereas the algorithm in \cite{ghadimi2013stochastic} is one timescale; and the gradient estimate is biased due to the Markovian noise and the two timescale update, and the bias needs to be explicitly characterized, whereas the gradient estimate in \cite{ghadimi2013stochastic} is unbiased. 

\subsection{Discussion on Technical Challenges}\label{sec:discussion}
In the following, we discuss the major challenges and highlight our major technical contributions in our analysis.  For the complete proof, we refer the readers to the appendix. 

For convenience, we define some notations. Let $O_t=(s_t,a_t,r_t,s_{t+1})$ be the observation at time $t$, define $\omega^*(\theta)=C^{-1}\mathbb{E}[\delta_{S,A,S'}(\theta)\phi_{S,A}]$. Here $\omega^*(\theta)$ can be interpreted as the limit of the fast timescale in \eqref{eq:omegaupdate} if we do not update the slow timescale parameter $\theta_t$ at all, and use a fixed $\theta$ in \eqref{eq:omegaupdate}. The tracking error is then defined to be $z_t=\omega_t-\omega^*(\theta_t)$. We further denote $G_{t+1}(\theta,\omega)=\delta_{t+1}(\theta)\phi_t-\gamma(\omega^\top  \phi_t)\hat{\phi}_{t+1}(\theta)$. Then, the update in the slow timescale in \eqref{eq:thetaupdate} is $\theta_{t+1}=\theta_t+\alpha G_{t+1}(\theta_t,\omega_t)$.

The major challenge in the analysis lies in analyzing the stochastic bias of the gradient estimate $G_{t+1}(\theta_t, \omega_t)$, which is introduced by: (1) the Markovian noise; and (2) the tracking error due to the two timescale update. Specifically, the bias in the gradient estimate can be decomposed as follows:
\begin{flalign}\label{eq:bias_decompose}
&\mE\left[G_{t+1}(\theta_t, \omega_t)+\frac{\nabla J(\theta)}{2}\right]\nn\\
&=\underbrace{\mE\left[G_{t+1}(\theta_t, \omega^*(\theta_t))+\frac{\nabla J(\theta)}{2}\right]}_{\text{ Markovian bias}}+\underbrace{\mE\left[G_{t+1}(\theta_t, \omega_t)-G_{t+1}(\theta_t, \omega^*(\theta_t))\right]}_{\text{tracking error}}.
\end{flalign}

For the first bias term in \eqref{eq:bias_decompose}, under the Markovian setting, $\theta_t$ is a function of $O_1,...,O_{t-1}$, and thus is dependent on $O_t$, which makes the  estimator biased. Our analysis employs a novel information theoretic technique \cite{bhandari2018finite} to bound the bias caused by this coupling. 

The second term in \eqref{eq:bias_decompose} is due to the tracking error $z_t\triangleq \omega_t-\omega^*(\theta_t)$ in the two timescale update, which is the most challenging part in our analysis. As $\theta_t$ changes at every time step, the limit of the fast timescale $\omega^*(\theta_t)$ also varies. Therefore, in the analysis of the tracking error, the change in $\theta_t$ also need to be taken into consideration. {The previous analyses of the tracking error were conducted using two approaches. The first approach involved using the size of a sample batch, resulting in a batch of size $\mathcal{O}(\epsilon^{-1})$. However, this approach yields a constant bound and is not applicable to our single-sample online fashion, where only a single sample is used at each step. The second approach, as seen in \cite{wang2020finite}, treats the update of $z_t$ as a single time-scale, bounds the tracking error individually, and only considers its convergence in terms of $T$. However, the convergence of the tracking error should also depend on the convergence of the slow time scale, i.e., $\theta$. In our analysis, we develop a framework that enables us to characterize the relationship between the two time scales and bound the tracking error in terms of both $T$ and $\theta$ simultaneously. Namely, we show that  
    $\frac{\sum^{T-1}_{t=0}\mathbb{E}[\Vert z_t\Vert ^2]}{T}\leq \mathcal{O}\Bigg(\frac{1}{T^{1-b}}+\frac{\log T}{T^b}+ \frac{\sum^{T-1}_{t=0}\mathbb{E}[\Vert \nabla J(\theta_t) \Vert ^2]}{T} \Bigg)$, where the convergence of $\nabla J(\theta)$ further introduces a tighter error bound on $z_t$, and results in an improved sample complexity.} 

\subsection{Discussion on Theoretical Results}
Our main theoretical results Theorem \ref{thm:main} and Corollary \ref{col:1} show that the Greedy-GQ algorithm converges to a stationary solution, i.e., $\theta_t \to \{ \theta: \nabla J(\theta)=0\}$. 

\begin{wrapfigure}{r}{0.5\textwidth}
  \begin{center}
    \includegraphics[width=0.48\textwidth]{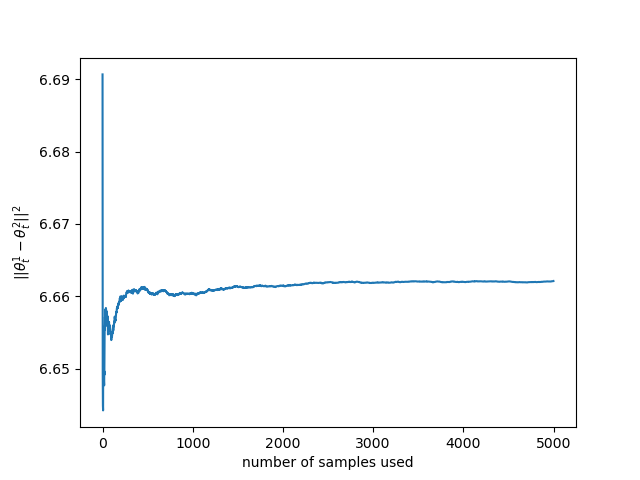}
  \end{center}
  \label{fig.difference}
\end{wrapfigure}

The stationary convergence result is not because of our analysis, instead, this is the nature of the Greedy-GQ algorithm. The objective function is a non-convex function. The Greedy-GQ algorithm can be viewed as a sub-gradient descent method with a two time-scale implementation, therefore we do not expect a global convergence result. 
To verify this, in Figure \ref{fig.difference}, we implement the Greedy-GQ algorithm with two different initializations $\theta^1_0,\theta^2_0$ under the Garnet problem $\mathcal{G}(10,5,10,5)$, and plot the difference between $\|\theta^1_t-\theta^2_t\|^2$ during the training. As the result shows, with different initializations,  $\theta^i_t$ converges to different stationary points. This implies that convergence to the global optimum cannot be achieved by the Greedy-GQ algorithm.

\color{black}

\section{Nested-loop Greedy-GQ}\label{sec:nest}
In this section, we design a novel nested-loop Greedy-GQ algorithm, and provide its finite-time error bound and sample complexity.

\subsection{Algorithm}
Instead of updating $\theta$ and $\omega$ simultaneously, the nested-loop Greedy-GQ algorithm consists of an inner loop and an outer loop. The slow timescale parameter $\theta$ is updated in the outer loop; and within the inner loop, the slow timescale parameter $\theta$ is kept fixed, and the fast timescale parameter $\omega$ is updated. Let $H_{t+1}(\theta,\omega)=(\delta_{t+1}(\theta)-\phi_t^\top \omega)\phi_t$. The algorithm is provided in Algorithm \ref{alg:nested}. Although batches of samples are used in the algorithm, the algorithm can still be implemented in an online and incremental fashion without having to store a batch of samples in the memory. 


\begin{algorithm}
\caption{Nested-loop Greedy-GQ}
\label{alg:nested}
\textbf{Input}:   $T$, $T_c$, $B$, $M$, $\alpha$, $\beta$, $\phi^{(i)}$ for $i=1,...,N$, $\pi_b$\\
\textbf{Initialization}: $\theta_0$,$w_0$
\begin{algorithmic}[1] 
\STATE {Choose $W\sim \text{Uniform}(0,1,...,T-1)$}
\FOR{$t=0,1,...,W-1$}
\STATE $w_{t,0}\leftarrow w_{t}$
\FOR {$t_c=0,1,...T_c-1$}
\STATE Generate $B$ samples
$O_j$, $j=(BT_c+M)t+Bt_c,...,(BT_c+M)t+Bt_c+B-1$
\STATE $w_{t,t_c+1} \leftarrow w_{t,t_c}+\frac{\beta}{B} \sum^B_{i=1} H_{(BT_c+M)t+Bt_c+i}(\theta_t,w_{t,t_c})$
\ENDFOR
\STATE $w_{t+1}\leftarrow w_{t,T_c-1}$
\STATE Generate $M$ samples $O_j$, $j=BT_c(t+1)+Mt,...,BT_c(t+1)+Mt+M-1$
\STATE $\theta_{t+1} \leftarrow  \theta_{t}+\frac{\alpha}{M}\sum^M_{i=1} G_{BT_c(t+1)+Mt+i}(\theta_t,w_t)$ 
\ENDFOR
\end{algorithmic}
\textbf{Output}: $\theta_W$
\end{algorithm}

\subsection{Finite-time Error Bound and Sample Complexity}
In this section, we present the finite-time error bound for the nested-loop Greedy-GQ algorithm. 

\begin{theorem}\label{nestcov}
Consider the nested-loop Greedy-GQ algorithm. Under both the Markovian and i.i.d.\ settings, if $\alpha<\frac{1}{K}$, $\beta<\frac{\lambda}{4}$, where $\lambda$ denotes the minimal eigenvalue of $C$, then
$
    \mathbb{E}[\Vert  \nabla J(\theta_W)\Vert ^2]\leq \mathcal{O}\left(e^{-T_c}+\frac{1}{T}+\frac{1}{B}+\frac{1}{M}\right). 
$
\end{theorem}

For the detailed proof under the Markovian setting, we refer the readers to the Section \ref{section:nested} in the appendix. The proof for the i.i.d. setting is similar, and is thus omitted. Here, we provide the order of the bound for simplicity. The explicit bound under the Markovian setting can be found in \eqref{eq:nestedmainresult} in the appendix. 
Theorem \ref{nestcov} shows that the nested-loop Greedy-GQ algorithm asymptotically converges to a neighborhood of a stationary point. In particular, the size of the neighborhood is $\mathcal{O}({1}/{M}+{1}/{B})$, which can be driven arbitrarily close to zero by choosing large batch sizes $M$ and $B$. 

We then derive the sample complexity of converging to an $\epsilon$-stationary point in the following corollary.
\begin{corollary}\label{nestsample}
Set $T, M, B =\mathcal{O}({1}/{\epsilon})$ and $T_c=\mathcal{O}(\log (\epsilon^{-1}))$, then the sample complexity of an $\epsilon$-stationary solution: $\mathbb{E}[\Vert  \nabla J(\theta_W)\Vert ^2]\leq \epsilon$ is  $\mathcal{O}\left({\log(\epsilon^{-1})}{\epsilon^{-2}}\right)$.
\end{corollary}

The sample complexity obtained in Corollary \ref{nestsample}  matches with the result in Theorem \ref{thm:main}.

\begin{remark}
In nested-loop Greedy-GQ, the fast timescale $w$ is updated in the inner loop using different samples from the samples for the update of $\theta$. Moreover, within the inner loop of $\omega$'s update, $\theta$ is kept fixed. This approach can simplify the analysis of the tracking error, compared to the vanilla Greedy-GQ algorithm which updates both $\theta$ and $\omega$ simultaneously. The sample complexity here also matches with that of the stochastic gradient descent algorithm for general non-convex problems in \cite{ghadimi2013stochastic}.
\end{remark}

\section{Experiments}
In this section, we conduct experiments on two RL problems: Garnet problem and the frozen lake problem, and compare the vanilla Greedy-GQ algorithm and its variants: the nested-loop one in this paper and the mini-batch one in \cite{xu2020sample}.

\subsection{Garnet Problem}
The first experiment is on the Garnet problem \cite{archibald1995generation}, which can be characterized by $\mathcal{G}(\vert  \mcs\vert  ,\vert  \mca\vert  ,b,N)$. Here $b$ is a branching factor specifying how many next states are possible for each state-action pair, and  these $b$ states are chosen uniformly at random. The transition probabilities are generated by sampling uniformly and randomly between 0 and 1. The parameter $N$ is the number of
features for linear function approximation. In our experiments, we generate a reward matrix uniformly and randomly between 0 and 1, and a feature matrix of dimension $N \times (\vert  \mcs\vert  \vert  \mca\vert  )$ randomly. In the nested-loop Greedy-GQ algorithm, we set $M=30$, $T_c=10$ and $B=5$. In the mini-batch Greedy-GQ algorithm, we set $B=30$. The step sizes are set as $\alpha=0.1$ and $\beta=0.5$, and the discount factor $\gamma=0.95$ in all the three algorithms.

We consider two sets of parameters: $\mathcal{G}(10,5,10,5)$ and $\mathcal{G}(8,10,5,4)$. In Figures \ref{Fig.gar1} and \ref{Fig.gar2}, we plot the minimum gradient norm v.s. the number of samples for all three algorithms using 40 Garnet MDP trajectories, i.e., at each time $t$, we plot $\min_{i\leq t} \Vert \nabla J(\theta_i)\Vert ^2$. The upper and lower envelopes of the curves correspond to the 95 and 5 percentiles of the 40 curves, respectively. We also plot the estimated variance of the stochastic update for three algorithms
along the iterations. Specifically, we query 100 Monte Carlo samples per iteration to estimate the squared error of $G_{t+1}(\theta_t,\omega_t)$, i.e.,  $\Vert G_{t+1}(\theta_t,\omega_t)-\nabla J(\theta_t)\Vert ^2$. It can be seen that both the nested-loop and mini-batch approaches can reduced the variance. There is no significant difference among the convergence rate of  the three algorithms. Different hyper-parameters, i.e., the batch size are also compared. We plot the norm of gradient v.s. the number of samples. The upper and lower envelop denotes the 95 and 5 percentiles of the 40 trajectories. The results show that the convergence rate is similar,  and mini-batch Greedy-GQ has a smaller variance with a larger batch size.

\begin{figure} [h]
\begin{tabular}{c c c}
  \includegraphics[width=0.3\textwidth]{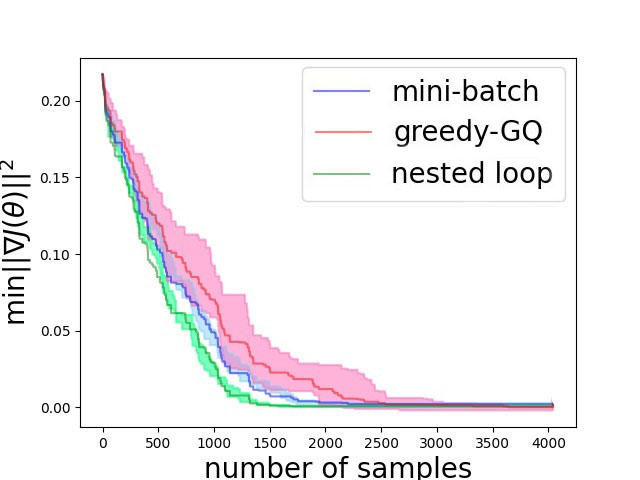} &
  \includegraphics[width=0.3\textwidth]{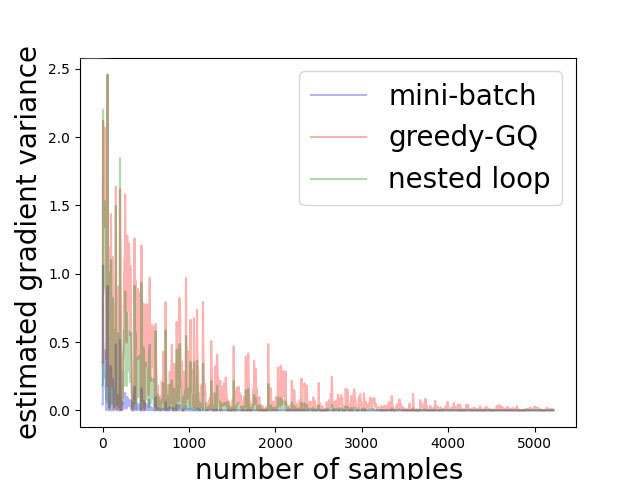}  &
  \includegraphics[width=0.3\textwidth]{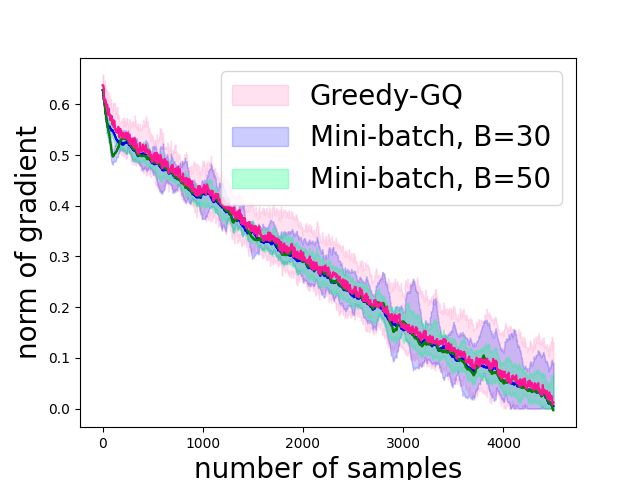}
\end{tabular}
\caption{Garnet Problem 1.}
\label{Fig.gar1}       
\end{figure}

\begin{figure} [h]
\begin{tabular}{c c c}
  \includegraphics[width=0.3\textwidth]{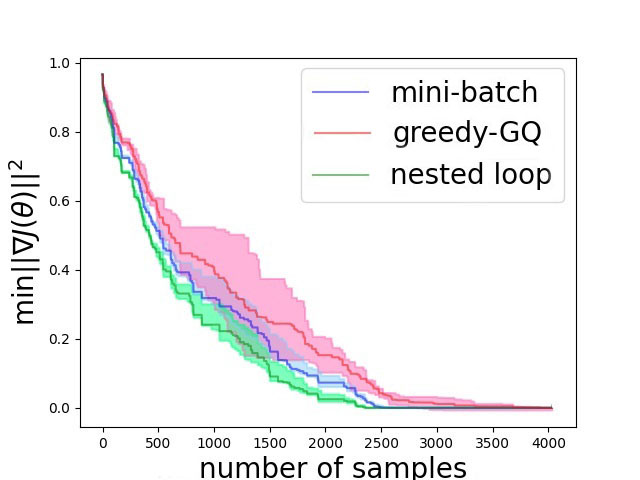} &
  \includegraphics[width=0.3\textwidth]{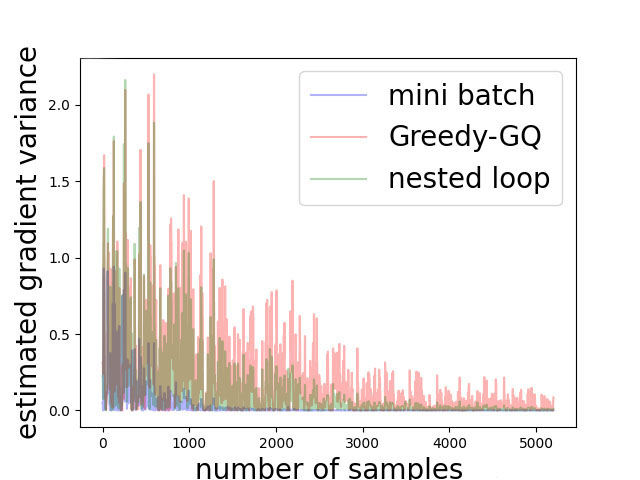}  &
  \includegraphics[width=0.3\textwidth]{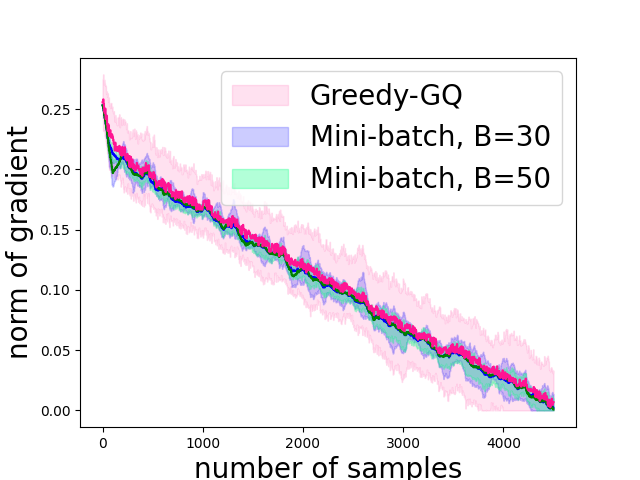}
\end{tabular}
\caption{Garnet Problem 2.}
\label{Fig.gar2}       
\end{figure}


\subsection{Frozen Lake Problem}
Our second experiment is on the frozen lake game  \cite{brockman2016openai}. We consider two sets of features with different number of features. Two random feature matrices $\Phi$ of dimension $4\times (\vert  \mcs\vert  \vert  \mca\vert  )$ and $5\times (\vert  \mcs\vert  \vert  \mca\vert  )$ are generated to linearly approximate the value function. In both problems, the agent follows the uniform behavior policy, i.e., it goes up, down, right or left with probability $\frac{1}{4}$. In the nested-loop Greedy-GQ algorithm, we set $M=30$, $T_c=10$ and $B=5$, while in the mini-batch Greedy-GQ algorithm, we set $B=30$. In all the three algorithms, the step sizes are set as $\alpha=0.1$ and $\beta=0.5$, and the discount factor $\gamma=0.95$.
We plot the minimum norm of the gradient and the estimated gradient variance as a function of the number of samples in Figures \ref{Fig.lake1} and \ref{Fig.lake2}. It can be seen that the nested-loop and mini-batch Greedy-GQ algorithms have  smaller gradient variance, and the convergence rate of the three algorithms are similar.   We also compare different batch sizes. The results also show that mini-batch can reduce the variance during the training.

\begin{figure} [h]
\begin{tabular}{c c c}
  \includegraphics[width=0.3\textwidth]{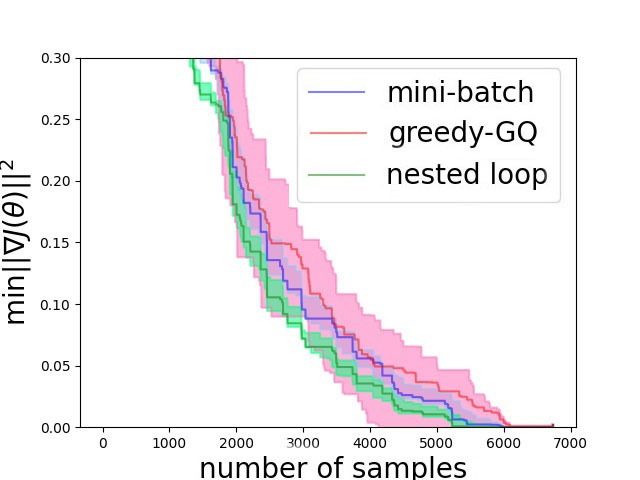} &
  \includegraphics[width=0.3\textwidth]{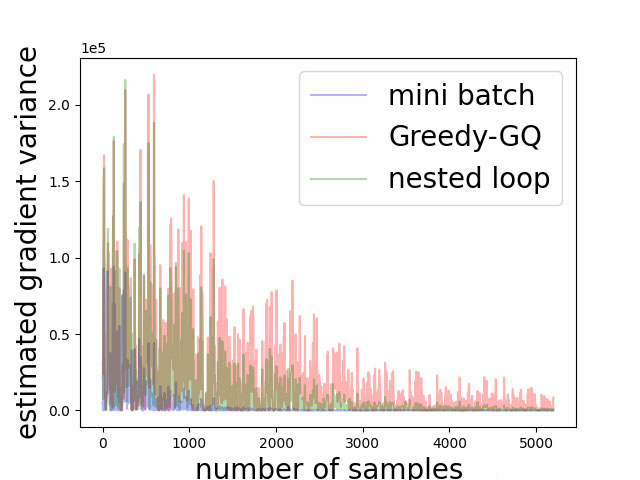}  &
  \includegraphics[width=0.3\textwidth]{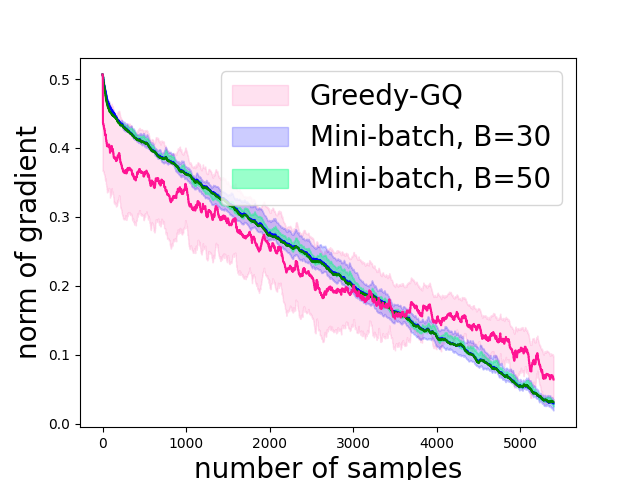}
\end{tabular}
\caption{Frozen Lake Problem 1.}
\label{Fig.lake1}       
\end{figure}

\begin{figure} [h]
\begin{tabular}{c c c}
  \includegraphics[width=0.3\textwidth]{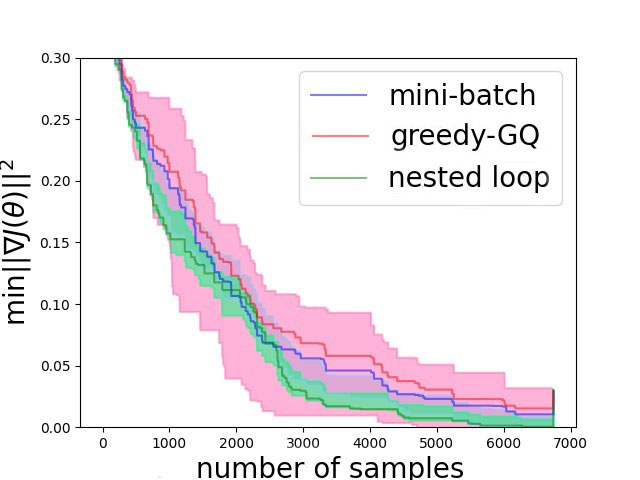} &
  \includegraphics[width=0.3\textwidth]{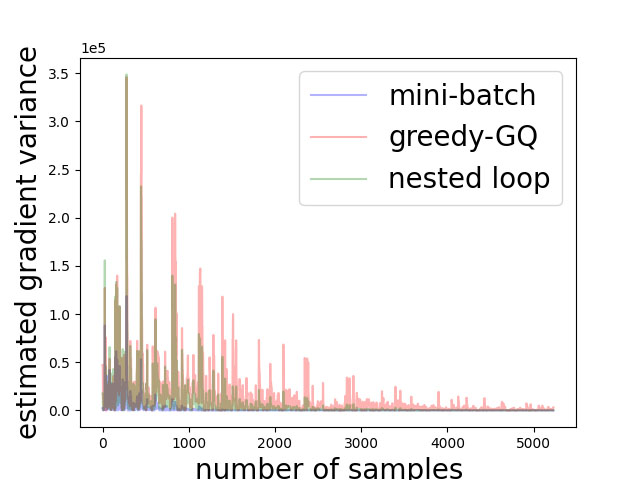}  &
  \includegraphics[width=0.3\textwidth]{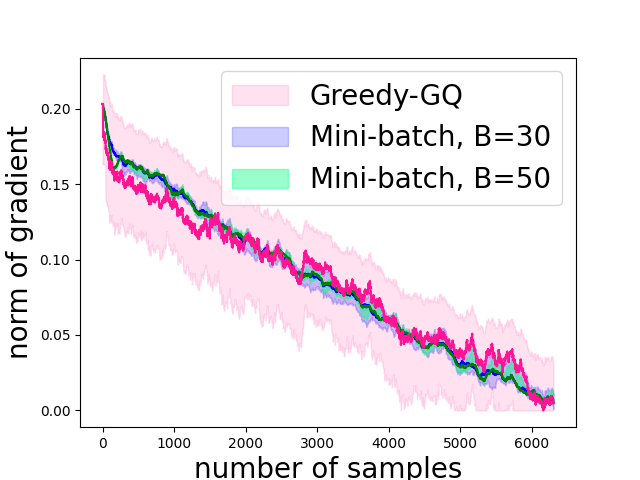}
\end{tabular}
\caption{Frozen Lake Problem 2.}
\label{Fig.lake2}       
\end{figure}

\section{Discussion}
In this section, we discuss the vanilla Greedy-GQ algorithm and its variants: nested-loop Greedy-GQ, mini-batch Greedy-GQ \cite{xu2020sample}, and variance reduction Greedy-GQ \cite{ma2020variance}, based on our theoretical and numerical results. First of all, from our numerical experiments, there is no significant difference in the convergence rate of the three algorithms. The variance of the gradient estimates of the mini-batch Greedy-GQ and nested-loop Greedy-GQ is smaller than the vanilla Greedy-GQ. Second, as can be seen from the theoretical bounds, there is a $\mathcal O(\log 1/\epsilon)$ factor improvement of the mini-batch Greedy-GQ and variance reduction Greedy-GQ than the vanilla Greedy-GQ and the nested-loop Greedy-GQ in the sample complexity, which however does not really appear to be the case from the numerical results. Therefore, such a gap of $\mathcal O(\log 1/\epsilon)$ is likely due to the analysis. Third, compared to the vanilla Greedy-GQ, where the update in $\theta$ is performed every time when a new sample comes in, the nested-loop and mini-batch methods update $\theta$ after a large batch of data (in Corollary \ref{nestsample} and Theorem 3 in \cite{xu2020sample}, the batch sizes needed are $\mathcal O(1/\epsilon)$ and $\mathcal O(\log(1/\epsilon)/\epsilon)$, respectively).  The vanilla Greedy-GQ has the advantage that it can be implemented in an online and incremental fashion, and can be stopped anytime to output the parameter $\theta$; however, to update the parameter $\theta$ once, the nested-loop and mini-batch Greedy-GQ methods need a big batch of data. Moreover, the vanilla Greedy-GQ has less number of hyper-parameter to tune in practice than the other two variants. Therefore, the vanilla Greedy-GQ is more convenient for practical online implementation.

\section{Conclusion}
In this paper, we developed the finite-time error bounds for the two timescale Greedy-GQ algorithm with linear function approximation under both i.i.d.\ and Markovian settings. We also proposed the nested-loop Greedy-GQ algorithm, and characterized its finite-time error bound and sample complexity. Including the mini-batch Greedy-GQ algorithm studied in \cite{xu2020sample}, all the three algorithms were shown to achieve the same sample complexity as the one of the stochastic gradient descent for a general non-convex optimization problem (up to a factor of $\log(1/\epsilon)$). The major technical contribution in this paper is a tight tracking error analysis which bound the tracking error in terms of the slow timescale parameter, and a novel analysis for non-convex optimization with two timescale updates and biased gradient. The tools and analysis developed in this paper can be used to improve the tracking error analysis and further the finite-time error bounds for a variety of two timescale RL algorithms. The theoretical understanding developed for the vanilla Greedy-GQ algorithm and its variants will provide useful insights for their application in practice.

\bibliography{sn-bibliography}

\begin{thebibliography}{10}

\bibitem{maei2010toward}
H.~R. Maei, C.~Szepesv{\'a}ri, S.~Bhatnagar, and R.~S. Sutton, ``Toward off-policy learning control with function approximation,'' in {\em Proc. International Conference on Machine Learning (ICML)}, 2010.

\bibitem{sutton2018reinforcement}
R.~S. Sutton and A.~G. Barto, {\em Reinforcement Learning: An Introduction, Second Edition}.
\newblock Cambridge, Massachusetts: The MIT Press, 2018.

\bibitem{watkins1992q}
C.~J. Watkins and P.~Dayan, ``Q-learning,'' {\em Machine learning}, vol.~8, no.~3-4, pp.~279--292, 1992.

\bibitem{baird1995residual}
L.~Baird, ``Residual algorithms: Reinforcement learning with function approximation,'' in {\em Machine Learning Proceedings}, (California), pp.~30--37, Elsevier, 1995.

\bibitem{gordon1996chattering}
G.~J. Gordon, ``Chattering in {SARSA} ($\lambda$),'' {\em CMU Learning Lab Technical Report}, 1996.

\bibitem{maei2011gradient}
H.~R. Maei, ``Gradient temporal-difference learning algorithms,'' {\em Thesis, University of Alberta}, 2011.

\bibitem{sutton2009fast}
R.~S. Sutton, H.~R. Maei, D.~Precup, S.~Bhatnagar, D.~Silver, C.~Szepesv{\'a}ri, and E.~Wiewiora, ``Fast gradient-descent methods for temporal-difference learning with linear function approximation,'' in {\em Proc. International Conference on Machine Learning (ICML)}, pp.~993--1000, 2009.

\bibitem{Sutton2009b}
R.~S. Sutton, H.~R. Maei, and C.~Szepesv{\'a}ri, ``A convergent $ {O} (n) $ temporal-difference algorithm for off-policy learning with linear function approximation,'' in {\em Proc. Advances in Neural Information Processing Systems (NIPS)}, pp.~1609--1616, 2009.

\bibitem{yu2017convergence}
H.~Yu, ``On convergence of some gradient-based temporal-differences algorithms for off-policy learning,'' {\em arXiv preprint arXiv:1712.09652}, 2017.

\bibitem{borkar2009stochastic}
V.~S. Borkar, {\em Stochastic approximation: a dynamical systems viewpoint}, vol.~48.
\newblock New York: Springer, 2009.

\bibitem{borkar2018concentration}
V.~S. Borkar and S.~Pattathil, ``Concentration bounds for two time scale stochastic approximation,'' in {\em Proc. Annu. Allerton Conf. Communication, Control and Computing}, pp.~504--511, IEEE, 2018.

\bibitem{karmakar2018two}
P.~Karmakar and S.~Bhatnagar, ``Two time-scale stochastic approximation with controlled {Markov} noise and off-policy temporal-difference learning,'' {\em Mathematics of Operations Research}, vol.~43, no.~1, pp.~130--151, 2018.

\bibitem{dalal2018finite}
G.~Dalal, B.~Sz{\"o}r{\'e}nyi, G.~Thoppe, and S.~Mannor, ``Finite sample analysis of two-timescale stochastic approximation with applications to reinforcement learning,'' {\em Proceedings of Machine Learning Research}, vol.~75, pp.~1--35, 2018.

\bibitem{wang2017finite}
Y.~Wang, W.~Chen, Y.~Liu, Z.-M. Ma, and T.-Y. Liu, ``Finite sample analysis of the {GTD} policy evaluation algorithms in markov setting,'' in {\em Proc. Advances in Neural Information Processing Systems (NIPS)}, pp.~5504--5513, 2017.

\bibitem{liu2015finite}
B.~Liu, J.~Liu, M.~Ghavamzadeh, S.~Mahadevan, and M.~Petrik, ``Finite-sample analysis of proximal gradient {TD} algorithms.,'' in {\em Proc. International Conference on Uncertainty in Artificial Intelligence (UAI)}, pp.~504--513, Citeseer, 2015.

\bibitem{gupta2019finite}
H.~Gupta, R.~Srikant, and L.~Ying, ``Finite-time performance bounds and adaptive learning rate selection for two time-scale reinforcement learning,'' in {\em Proc. Advances in Neural Information Processing Systems (NeurIPS)}, pp.~4706--4715, 2019.

\bibitem{xu2019two}
T.~Xu, S.~Zou, and Y.~Liang, ``Two time-scale off-policy {TD} learning: Non-asymptotic analysis over {Markovian} samples,'' in {\em Proc. Advances in Neural Information Processing Systems (NeurIPS)}, pp.~10633--10643, 2019.

\bibitem{xu2020sample}
T.~Xu and Y.~Liang, ``Sample complexity bounds for two timescale value-based reinforcement learning algorithms,'' in {\em International Conference on Artificial Intelligence and Statistics}, pp.~811--819, PMLR, 2021.

\bibitem{ma2020variance}
S.~Ma, Y.~Zhou, and S.~Zhou, ``Variance-reduced off-policy {TDC} learning: Non-asymptotic convergence analysis,'' {\em arXiv preprint arXiv:2010.13272}, 2020.

\bibitem{konda2004convergence}
V.~R. Konda, J.~N. Tsitsiklis, {\em et~al.}, ``Convergence rate of linear two-time-scale stochastic approximation,'' {\em The Annals of Applied Probability}, vol.~14, no.~2, pp.~796--819, 2004.

\bibitem{dalal2020tale}
G.~Dalal, B.~Szorenyi, and G.~Thoppe, ``A tale of two-timescale reinforcement learning with the tightest finite-time bound,'' in {\em Proc. AAAI Conference on Artificial Intelligence (AAAI)}, pp.~3701--3708, 2020.

\bibitem{kaledin2020finite}
M.~Kaledin, E.~Moulines, A.~Naumov, V.~Tadic, and H.-T. Wai, ``Finite time analysis of linear two-timescale stochastic approximation with {Markovian} noise,'' {\em arXiv preprint arXiv:2002.01268}, 2020.

\bibitem{ghadimi2013stochastic}
S.~Ghadimi and G.~Lan, ``Stochastic first- and zeroth-order methods for nonconvex stochastic programming,'' {\em SIAM Journal on Optimization}, vol.~23, no.~4, pp.~2341--2368, 2013.

\bibitem{mokkadem2006convergence}
A.~Mokkadem, M.~Pelletier, {\em et~al.}, ``Convergence rate and averaging of nonlinear two-time-scale stochastic approximation algorithms,'' {\em The Annals of Applied Probability}, vol.~16, no.~3, pp.~1671--1702, 2006.

\bibitem{doan2021nonlinear}
T.~T. Doan, ``Nonlinear two-time-scale stochastic approximation: Convergence and finite-time performance,'' in {\em Learning for Dynamics and Control}, pp.~47--47, PMLR, 2021.

\bibitem{srikant2019}
R.~Srikant and L.~Ying, ``Finite-time error bounds for linear stochastic approximation and {TD} learning,'' in {\em Proc. Annual Conference on Learning Theory (CoLT)}, 2019.

\bibitem{lakshminarayanan2018linear}
C.~Lakshminarayanan and C.~Szepesvari, ``Linear stochastic approximation: {H}ow far does constant step-size and iterate averaging go?,'' in {\em International Conference on Artificial Intelligence and Statistics}, pp.~1347--1355, 2018.

\bibitem{bhandari2018finite}
J.~Bhandari, D.~Russo, and R.~Singal, ``A finite time analysis of temporal difference learning with linear function approximation,'' {\em arXiv preprint arXiv:1806.02450}, 2018.

\bibitem{Dalal2018a}
G.~Dalal, B.~Szrnyi, G.~Thoppe, and S.~Mannor, ``Finite sample analyses for {TD}(0) with function approximation,'' in {\em Proc. AAAI Conference on Artificial Intelligence (AAAI)}, pp.~6144--6160, 2018.

\bibitem{sun2020finite}
J.~Sun, G.~Wang, G.~B. Giannakis, Q.~Yang, and Z.~Yang, ``Finite-time analysis of decentralized temporal-difference learning with linear function approximation,'' in {\em Proc. International Conference on Artifical Intelligence and Statistics (AISTATS)}, pp.~4485--4495, PMLR, 2020.

\bibitem{zou2019finite}
S.~Zou, T.~Xu, and Y.~Liang, ``Finite-sample analysis for {SARSA} with linear function approximation,'' in {\em Proc. Advances in Neural Information Processing Systems (NeurIPS)}, pp.~8665--8675, 2019.

\bibitem{cai2019neural}
Q.~Cai, Z.~Yang, J.~D. Lee, and Z.~Wang, ``Neural temporal-difference learning converges to global optima,'' in {\em Proc. Advances in Neural Information Processing Systems (NeurIPS)}, pp.~11312--11322, 2019.

\bibitem{xu2020finite}
P.~Xu and Q.~Gu, ``A finite-time analysis of {Q}-learning with neural network function approximation,'' in {\em Proc. International Conference on Machine Learning (ICML)}, pp.~10555--10565, 2020.

\bibitem{bhatnagar2009convergent}
S.~Bhatnagar, D.~Precup, D.~Silver, R.~S. Sutton, H.~Maei, and C.~Szepesv{\'a}ri, ``Convergent temporal-difference learning with arbitrary smooth function approximation,'' {\em Proc. Advances in Neural Information Processing Systems (NIPS)}, vol.~22, pp.~1204--1212, 2009.

\bibitem{ma2021greedy}
S.~Ma, Z.~Chen, Y.~Zhou, and S.~Zou, ``Greedy-gq with variance reduction: Finite-time analysis and improved complexity,'' in {\em The International Conference on Learning Representations (ICLR)}, 2021.

\bibitem{wang2020finite}
Y.~Wang and S.~Zou, ``Finite-sample analysis of {Greedy-GQ} with linear function approximation under {Markovian} noise,'' in {\em Proc. Uncertainty in Artificial Intelligence (UAI)}, pp.~11--20, PMLR, 2020.

\bibitem{archibald1995generation}
T.~Archibald, K.~McKinnon, and L.~Thomas, ``{On the generation of Markov decision processes},'' {\em Journal of the Operational Research Society}, vol.~46, no.~3, pp.~354--361, 1995.

\bibitem{brockman2016openai}
G.~Brockman, V.~Cheung, L.~Pettersson, J.~Schneider, J.~Schulman, J.~Tang, and W.~Zaremba, ``{OpenAI Gym},'' {\em arXiv preprint arXiv:1606.01540}, 2016.

\end{thebibliography}
\bibliographystyle{ieeetr}


\newpage
\begin{appendices}

\section{Analysis for Vanilla Greedy-GQ}\label{app:A}
In the following proof, $\Vert a\Vert $ denotes the $\ell_2$ norm if $a$ is a vector; and $\Vert A\Vert $ denotes the operator norm if $A$ is a matrix.
For technical convenience, we impose a projection step on both the updates of $\theta$ and $\omega$ with radius $R$: for any $t$, $\Vert \theta_t\Vert \leq R$ and $\Vert \omega_t\Vert \leq R$. The projection step is necessary to guarantee the stability of the algorithm. The approach developed in \cite{srikant2019} which bounds the parameter using its retrospective copy several time steps back, is not applicable here due to the nonlinear structure of Greedy-GQ.



We first show that the objective function $J(\theta)$ is $K$-smooth for $\theta \in \{\theta: \Vert \theta\Vert \leq R\}$.
\begin{Lemma}\label{lemma:Lsmooth}
	$J(\theta)$ is $K$-smooth: 
	\begin{flalign*}
		\Vert \nabla J(\theta_1)-\nabla J(\theta_2)\Vert  \leq K\vert  \vert  \theta_1-\theta_2\vert  \vert  , \forall \Vert \theta_1\Vert ,\Vert \theta_2 \Vert  \leq R,
	\end{flalign*} 
	where 
$
 K=2\gamma{\lambda^{-1}}\big((k_1\vert  \mca\vert  R+1)(1+\gamma+\gamma Rk_1\vert  \mca\vert  )+\vert  \mca\vert  (r_{\max}+R+\gamma R)( 2k_1+ k_2R) \big).
$
\end{Lemma}
\begin{proof}
	
	It  follows that  
	\begin{align*}
		&\nabla J\left (\theta_1\right )-\nabla J\left (\theta_2\right )\nn\\
		&=2\nabla \left (\mathbb{E}_{\mu}\left [\delta_{S,A,S'}\left (\theta_1\right )\phi_{S,A}\right ]\right ) C^{-1}\mathbb{E}_{\mu}\left [\delta_{S,A,S'}\left (\theta_1\right )\phi_{S,A}\right ]\nn\\
		&\quad-2\nabla \left (\mathbb{E}_{\mu}\left [\delta_{S,A,S'}\left (\theta_2\right )\phi_{S,A}\right ]\right ) C^{-1}\mathbb{E}_{\mu}\left [\delta_{S,A,S'}\left (\theta_2\right )\phi_{S,A}\right ]\nn\\
		&=2\nabla \left (\mathbb{E}_{\mu}\left [\delta_{S,A,S'}\left (\theta_1\right )\phi_{S,A}\right ]\right ) C^{-1}\mathbb{E}_{\mu}\left [\delta_{S,A,S'}\left (\theta_1\right )\phi_{S,A}\right ]\nn\\
		&\quad-2\nabla \left (\mathbb{E}_{\mu}\left [\delta_{S,A,S'}\left (\theta_1\right )\phi_{S,A}\right ]\right ) C^{-1}\mathbb{E}_{\mu}\left [\delta_{S,A,S'}\left (\theta_2\right )\phi_{S,A}\right ]\nn\\
		&\quad+2\nabla \left (\mathbb{E}_{\mu}\left [\delta_{S,A,S'}\left (\theta_1\right )\phi_{S,A}\right ]\right ) C^{-1}\mathbb{E}_{\mu}\left [\delta_{S,A,S'}\left (\theta_2\right )\phi_{S,A}\right ]\nn\\
		&\quad-2\nabla \left (\mathbb{E}_{\mu}\left [\delta_{S,A,S'}\left (\theta_2\right )\phi_{S,A}\right ]\right ) C^{-1}\mathbb{E}_{\mu}\left [\delta_{S,A,S'}\left (\theta_2\right )\phi_{S,A}\right ].
	\end{align*}
Since $C^{-1}$ is positive definite, thus  it suffices to show both $\nabla \left (\mathbb{E}_{\mu}\left [\delta_{S,A,S'}\left (\theta\right )\phi_{S,A}\right ]\right )$ and $\mathbb{E}_{\mu}\left [\delta_{S,A,S'}\left (\theta\right )\phi_{S,A}\right ]$ are Lipschitz in $\theta$ and are bounded.
	
It is straightforward to see that 
	\begin{align}\label{eq:30}
		\Vert \mathbb{E}_{\mu}\left [\delta_{S,A,S'}\left (\theta\right )\phi_{S,A}\right ]\Vert \leq r_{\max}+(1+\gamma) R,
	\end{align}
	and
$
		\Vert \nabla \mathbb{E}_{\mu}\left [\delta_{S,A,S'}\left (\theta\right )\phi_{S,A}\right ] \Vert \leq 1+\gamma(k_1\vert  \mca\vert  R+1).
$
We further have that\begin{small}
	\begin{align}\label{eq:33}
		&\Vert \nabla \left (\mathbb{E}_{\mu}\left [\delta_{S,A,S'}\left (\theta_1\right )\phi_{S,A}\right ]\right )-\nabla \left (\mathbb{E}_{\mu}\left [\delta_{S,A,S'}\left (\theta_2\right )\phi_{S,A}\right ]\right )\Vert \nn\\
		&= \gamma\bigg\Vert \mathbb{E}_{\mu}\bigg [\sum_{a\in\mathcal{A}}\bigg(  \nabla\left (\pi_{\theta_1}\left (a\vert  S'\right )\right)\theta_1^\top \phi_{S',a}-\nabla\left (\pi_{\theta_2}\left (a\vert  S'\right )\right)\nn\\
		&\quad\cdot\theta_2^\top \phi_{S',a}+\pi_{\theta_1}\left (a\vert  S'\right )\phi_{S',a}-\pi_{\theta_2}\left (a\vert  S'\right )\phi_{S',a}\bigg)\phi_{S,A}^\top\bigg ]\bigg\Vert \nn\\
		&= \gamma\bigg\Vert \mathbb{E}_{\mu}\bigg [\sum_{a\in\mathcal{A}}\bigg(  \nabla\left (\pi_{\theta_1}\left (a\vert  S'\right )-\pi_{\theta_2}\left (a\vert  S'\right )\right)\theta_1^\top \phi_{S',a}\nn\\
		&\quad+\nabla\left (\pi_{\theta_2}\left (a\vert  S'\right )\right)(\theta_1-\theta_2)^\top \phi_{S',a}\bigg)\phi_{S,A}^\top\bigg ]\bigg\Vert \\
		&\quad+\gamma\big\Vert \mathbb{E}_{\mu}\big [\big (\sum_{a\in \mca} \big ( \pi_{\theta_1}\left (a\vert  S'\right )\phi_{S',a}-\pi_{\theta_2}\left (a\vert  S'\right )\phi_{S',a}\big)\big )\phi_{S,A}^\top\big ]\big\Vert \nn\\
		&\leq \gamma\vert  \mca\vert  \left(2k_1+ k_2 R \right)\Vert \theta_1-\theta_2\Vert ,
	\end{align}\end{small}
	which is from Assumption \ref{assump:policy}.
	Following similar steps, we can also show  that $\mathbb{E}_{\mu}\left [\delta_{S,A,S'}\left (\theta\right )\phi_{S,A}\right ]$ is Lipschitz in $\theta$:
	\begin{align}\label{eq:34}
		\Vert \mathbb{E}_{\mu}\left [\delta_{S,A,S'}\left (\theta_1\right )\phi_{S,A}\right ]-\mathbb{E}_{\mu}\left [\delta_{S,A,S'}\left (\theta_2\right )\phi_{S,A}\right ]\Vert \nn\\
		\leq \left (\gamma(\vert  \mca\vert  k_1R+1)+1\right ) \Vert \theta_1-\theta_2\Vert .
	\end{align}
	
	Combining \eqref{eq:33} and \eqref{eq:34} concludes the proof.
\end{proof}

Recall the definition of $G_{t+1}(\theta, \omega)$ in Section \ref{sec:discussion}.
The following Lemma shows that $G_{t+1}(\theta, \omega)$ is Lipschitz in $\omega$, and $G_{t+1}(\theta, \omega^*(\theta))$ is Lipschitz in $\theta$.
\begin{Lemma}\label{Lemma:Lipschitz}
	For any $w_1,w_2$, 
$
		\Vert G_{t+1}(\theta,\omega_1)-G_{t+1}(\theta,\omega_2)\Vert 
		\leq \gamma(\vert  \mca\vert  Rk_1+1)\Vert \omega_1-\omega_2\Vert ,
$
	and for any $\theta_1,\theta_2\in \{\theta:\Vert \theta\Vert \leq R\}$,
	\begin{align}\label{eq:G^*lip}
		\hspace{-0.2cm}\Vert G_{t+1}(\theta_1,\omega^*(\theta_1))-G_{t+1}(\theta_2,\omega^*(\theta_2))\Vert \leq k_3\Vert \theta_1-\theta_2 \Vert ,
	\end{align}
	where $k_3=(1+\gamma+\gamma R\vert  \mca\vert  k_1+\gamma\frac{1}{\lambda}\vert  \mca\vert  (2k_1+k_2R)(r_{\max}+\gamma R+R)+\gamma \frac{1}{\lambda}(1+\vert  \mca\vert  Rk_1)(1+\gamma+\gamma R\vert  \mca\vert  k_1)).$
\end{Lemma}
\begin{proof}
	Under Assumption \ref{assump:policy}, it can be easily shown that
	\begin{flalign}\label{eq:39}
		\Vert \hat{\phi}_{t+1}(\theta) \Vert \leq \vert  \mca\vert  Rk_1+1.
	\end{flalign}
	It then follows that for any $\omega_1$ and $\omega_2$,
	\begin{align*}
		&\Vert G_{t+1}(\theta,\omega_1)-G_{t+1}(\theta,\omega_2)\Vert 
		\leq \gamma (\vert  \mca\vert  Rk_1+1)\Vert \omega_1-\omega_2\Vert .
	\end{align*}
	
To show that $G_{t+1}(\theta,\omega^*(\theta))$ is Lipschitz in $\theta$, we first show that $\hat{\phi}_{t+1}(\theta)$ is Lipschitz in $\theta$ following similar steps as those in \eqref{eq:33}:
	\begin{align}
		&\Vert \hat{\phi}_{t+1}(\theta_1)-\hat{\phi}_{t+1}(\theta_2)\Vert 
		\leq \vert   \mathcal{A}\vert   (2k_1  +k_2R)\Vert \theta_1-\theta_2 \Vert .
	\end{align}
	 We have that\begin{small}
	\begin{align}
		&\Vert G_{t+1}(\theta_1,\omega^*(\theta_1))-G_{t+1}(\theta_2,\omega^*(\theta_2)) \Vert \nn\\
		&\leq \vert  \delta_{t+1}(\theta_1)-\delta_{t+1}(\theta_2) \vert  
		+\gamma\Vert (\omega^*(\theta_2))^\top\phi_t\hat{\phi}_{t+1}(\theta_2)\nn\\
		&\quad-(\omega^*(\theta_1))^\top\phi_t\hat{\phi}_{t+1}(\theta_1) \Vert \nn\\
		&\overset{(a)}{\leq}\gamma\Vert (\omega^*(\theta_2))^\top \phi_t\hat{\phi}_{t+1}(\theta_2)-(\omega^*(\theta_1))^\top\phi_t\hat{\phi}_{t+1}(\theta_1)\nn\\
		&\quad-(\omega^*(\theta_1))^\top\phi_t\hat{\phi}_{t+1}(\theta_2)+(\omega^*(\theta_1))^\top\phi_t\hat{\phi}_{t+1}(\theta_2) \Vert \nn\\
		&\quad+(1+\gamma+\gamma R\vert  \mathcal{A}\vert  k_1)\Vert \theta_1-\theta_2 \Vert \nn\\
		&\leq \gamma(1+\vert  \mca\vert  Rk_1)\Vert \omega^*(\theta_2)-\omega^*(\theta_1) \Vert \nn\\
		&\quad+\gamma\Vert \omega^*(\theta_1) \Vert \Vert \hat{\phi}_{t+1}(\theta_1)-\hat{\phi}_{t+1}(\theta_2)\ \Vert \nn\\
		&\quad+\gamma(1+R\vert  \mathcal{A}\vert  k_1)\Vert \theta_1-\theta_2 \Vert +\Vert \theta_1-\theta_2 \Vert \nn\\
		&\overset{(b)}{\leq } \bigg(\left(1+\frac{\gamma }{\lambda}(1+\vert  \mca\vert  Rk_1)\right)(1+\gamma+\gamma R\vert  \mca\vert  k_1)\nn\\
		&\quad+\frac{\gamma}{\lambda}\vert  \mca\vert  (2k_1+k_2R)(r_{\max}+\gamma R+R)\bigg) \Vert \theta_1-\theta_2 \Vert ,
	\end{align}\end{small}
	where $(a)$ can be shown following steps similar to those  in \eqref{eq:34}, while  $(b)$ can be shown using
	\begin{align}\label{eq:w*lip}
		\Vert \omega^*(\theta_2)-\omega^*(\theta_1)\Vert \leq  \frac{(1+\gamma+\gamma R\vert  \mca\vert  k_1)} {\lambda}\Vert \theta_1-\theta_2\Vert ,
	\end{align}
	and $
		\Vert \omega^*(\theta)\Vert \leq \frac{1}{\lambda}(r_{\max}+\gamma R+R).
	$
\end{proof}
%


Since $J(\theta)$ is $K$-smooth, by Taylor expansion we have that
\begin{align}\label{eq:Jksmooth}
     &J(\theta_{t+1})    \leq J(\theta_t) +\left\langle \nabla J(\theta_t), \theta_{t+1}-\theta_t\right\rangle  + \frac{K}{2} \Vert  \theta_{t+1}-\theta_t\Vert ^2\nn\\
    &=J(\theta_t)-\alpha\big\langle \nabla J(\theta_t),-G_{t+1}(\theta_t, \omega_t)+G_{t+1}(\theta_t, \omega^*(\theta_t)) \big\rangle\nn\\
    &\quad+\frac{\alpha}{2} \langle \nabla J(\theta_t), {\nabla J(\theta_t)} +2G_{t+1}(\theta_t, \omega^*(\theta_t)) \rangle\nn\\
    &\quad -\frac{\alpha}{2}\vert  \vert  \nabla J(\theta_t)\vert  \vert  ^2+\frac{K}{2} \alpha^2\vert  \vert  G_{t+1}(\theta_t,\omega_t)\vert  \vert  ^2\nn\\
    &\leq J(\theta_t) +\alpha \gamma\Vert \nabla J(\theta_t) \Vert (1+\vert  \mca\vert  Rk_1)\Vert \omega^*(\theta_t)-\omega_t \Vert \nn\\
    &\quad+\frac{\alpha}{2} \langle \nabla J(\theta_t), {\nabla J(\theta_t)}+2G_{t+1}(\theta_t, \omega^*(\theta_t)) \rangle \nn\\
    &\quad-\frac{\alpha}{2}\vert  \vert  \nabla J(\theta_t)\vert  \vert  ^2+\frac{K}{2} \alpha^2\vert  \vert  G_{t+1}(\theta_t,\omega_t)\vert  \vert  ^2,
\end{align}
where the last inequality follows from Lemma \ref{Lemma:Lipschitz}.

Re-arranging the terms in \eqref{eq:Jksmooth}, summing up w.r.t. $t$ from 0 to $T-1$, taking the expectation and applying Cauchy's inequality implies that
\begin{small}
\begin{align}\label{eq:main}
    &\sum^{T-1}_{t=0} \frac{\alpha}{2} \mathbb{E}[\Vert \nabla J(\theta_t) \Vert ^2]\nn\\
    &\leq J(\theta_0)-J(\theta_{T})+ \gamma\alpha(1+\vert  \mca\vert  Rk_1)\sqrt{\sum^{T-1}_{t=0} \mathbb{E}[\Vert \nabla J(\theta_t)\Vert ^2]}\nn\\
    &\quad\cdot\sqrt{\sum^{T-1}_{t=0}\mathbb{E}[\Vert \omega^*(\theta_t)-\omega_t\Vert ^2]} +\frac{K}{2}\sum^{T-1}_{t=0}\alpha^2 \mathbb{E}[\Vert  G_{t+1}(\theta_t,\omega_t)\Vert ^2]\nn\\
    &\quad+\sum^{T-1}_{t=0}\frac{\alpha}{2} \mathbb{E}\left[\left\langle \nabla J(\theta_t),{\nabla J(\theta_t)}+2G_{t+1}(\theta_t, \omega^*(\theta_t)) \right\rangle\right].
\end{align}\end{small}

We then provide the bounds on $\mathbb{E}[\Vert \omega^*(\theta_t)-\omega_t\Vert ^2]$ and $\mathbb{E}\left[\left\langle \nabla J(\theta_t),{\nabla J(\theta_t)}/{2}+G_{t+1}(\theta_t, \omega^*(\theta_t)) \right\rangle\right]$, which we refer to as ``tracking error" and ``stochastic bias". We 
 define $\zeta(\theta, O_t)\triangleq\langle \nabla J(\theta), \frac{\nabla J(\theta)}{2}+G_{t+1}(\theta, \omega^*(\theta)) \rangle$, then $\mathbb{E}_{\mu}[\zeta(\theta, O_t)]=0$ for any fixed $\theta$ when $O_t\sim\mu$ (which doesn't hold under the Markovian setting). 
In the following lemma, we provide an upper bound on $\mE[\zeta(\theta, O_t)]$.
\begin{Lemma}\label{thm:zeta}
\textbf{Stochastic Bias.}	Let $\tau_{\alpha}\triangleq \min \left\{k : m\rho^k \leq \alpha \right\}$. If $t \leq \tau_{\alpha}$, then 
$
		\mathbb{E}[\zeta(\theta_t,O_t)] \leq k_{\zeta},
$
	and if $t > \tau_{\alpha}$, then
\begin{align}
		\mathbb{E}[\zeta(\theta_t, O_t)]\leq  k_{\zeta}\alpha+c_{\zeta}(c_{f_1}+c_{g_1})\tau_{\alpha}\alpha,
\end{align}
	where $c_{\zeta}=2\gamma(1+k_1\vert  \mca\vert  R)\frac{1}{\lambda}(r_{\max}+R+\gamma R)(\frac{K}{2}+k_3)+K(r_{\max}+R+\gamma R)( \frac{2\gamma}{\lambda}(1+k_1\vert  \mca\vert  R)+1)$ and $k_{\zeta}=4\gamma(1+k_1R\vert  \mca\vert  )\frac{1}{\lambda}(r_{\max}+R+\gamma R)^2(2\gamma(1+k_1\vert  \mca\vert  R)\frac{1}{\lambda}+1)$.
\end{Lemma}

\begin{proof}
	For any $\theta_1$ and $\theta_2$, it follows that 
	\begin{align}\label{eq:52} 
		&\vert  \zeta(\theta_1,O_t)-\zeta(\theta_2,O_t)\vert  \nn\\
		&=\frac{1}{2}\vert  \left\langle \nabla J(\theta_1),  {\nabla J(\theta_1)} +2G_{t+1}(\theta_1, \omega^*(\theta_1)) \right\rangle\nn\\
		&\quad-\left\langle \nabla J(\theta_1), {\nabla J(\theta_2)}+2G_{t+1}(\theta_2, \omega^*(\theta_2)) \right\rangle\\
		&\quad+\left\langle \nabla J(\theta_1)-\nabla J(\theta_2), {\nabla J(\theta_2)} +2G_{t+1}(\theta_2, \omega^*(\theta_2)) \right\rangle
		\vert  .\nn
	\end{align}
	By Lemma \ref{lemma:Lsmooth},  $\zeta(\theta,O_t)$ is also Lipschitz in $\theta$: $
		\vert  \zeta(\theta_1,O_t)-\zeta(\theta_2,O_t)\vert   \leq c_{\zeta} \Vert \theta_1-\theta_2 \Vert ,
	$
	where 
$
		c_{\zeta}=2\gamma(1+k_1\vert  \mca\vert  R)\frac{1}{\lambda}(r_{\max}+R+\gamma R)(\frac{K}{2}+k_3)
		+K(r_{\max}+R+\gamma R)(\gamma \frac{1}{\lambda}(1+k_1\vert  \mca\vert  R)+1
		+\gamma \frac{1}{\lambda}(1+Rk_1\vert  \mca\vert  )).
$
	Thus from \eqref{eq:52}, it follows that for any $\tau \geq 0$,
	\begin{align}\label{eq:55}
		&\vert  \zeta(\theta_t,O_t)-\zeta(\theta_{t-\tau},O_t)\vert   \leq  c_{\zeta} \Vert \theta_t-\theta_{t-\tau} \Vert \\
		&\leq c_\zeta \sum^{t-1}_{k=t-\tau}\alpha \Vert G_{k+1}(\theta_k,\omega_k)\Vert \leq  c_{\zeta}(c_{f_1}+c_{g_1})\sum^{t-1}_{k=t-\tau}\alpha,\nn
	\end{align} 
	where $c_{f_1}=r_{\max}+(1+\gamma)R+\frac{\gamma}{\lambda}(r_{\max}+(1 + \gamma)R)(1+R\vert  \mathcal{A}\vert  k_1)$, $c_{g_1}=2\gamma R(1+R\vert  \mathcal{A}\vert  k_1)$ and $\Vert G_{k+1}(\theta_k,\omega_k)\Vert \leq c_{f_1}+c_{g_1}$.
	
	We define an independent random variable $\hat O=(\hat S,\hat A,\hat R,\hat S')$, where $(\hat S,\hat A)\sim\mu$, $\hat S'$ is the subsequent state and $\hat R$ is the reward. Then $\mathbb{E}[\zeta(\theta_{t-\tau},\hat O)]=0$ by the fact that $\mathbb{E}_{\mu}[{G_{t+1}(\theta,\omega^*(\theta))}]=-\frac{1}{2}\nabla J(\theta)$.
	Thus for any $\tau\leq t$, 
	\begin{align*}
		\mathbb{E}[\zeta(\theta_{t-\tau},O_t)] &\leq \vert  \mathbb{E}[\zeta(\theta_{t-\tau},O_t)]-\mathbb{E}[\zeta(\theta_{t-\tau},\hat{O})]\vert  \leq k_{\zeta}m\rho^{\tau},
	\end{align*}
	which follows from Assumption \ref{ass:1}, and $k_{\zeta}=4\gamma(1+k_1R\vert  \mca\vert  )\frac{1}{\lambda}(r_{\max}+R+\gamma R)^2(2\gamma(1+k_1\vert  \mca\vert  R)\frac{1}{\lambda}+1)$.
	
	If $t \leq \tau_{\alpha}$, the conclusion follows from the fact that $\vert  \zeta(\theta,O_t)\vert  \leq k_{\zeta}$.
%
%
	If $t > \tau_{\alpha}$, we choose $\tau=\tau_{\alpha}$, and then
$
		\mathbb{E}[\zeta(\theta_t, O_t)]\leq \mathbb{E}[\zeta(\theta_{t-\tau_{\alpha}},O_t)]+c_{\zeta}(c_{f_1}+c_{g_1})\sum^{t-1}_{k=t-\tau_{\alpha}}\alpha
		\leq k_{\zeta}\alpha+c_{\zeta}(c_{f_1}+c_{g_1})\tau_{\alpha}\alpha.
$
\end{proof}

The tracking error can be bounded in the following lemma.
\begin{Lemma}\label{lemma:tracking}\textbf{Tracking error.} (proof in Section \ref{app:lemmas})
\begin{align}
    &\frac{\sum^{T-1}_{t=0}\mathbb{E}[\Vert z_t\Vert ^2]}{T}\leq \frac{2Q_T}{T}+\frac{32}{1-e^{-2\lambda\beta}}\frac{ \Vert R_2\Vert ^2}{\lambda\beta}\nn\\
&\quad +    \frac{8\alpha^2}{\lambda^3\beta} \frac{(1+\gamma+\gamma k_1R\vert  \mca\vert  )^2}{1-e^{-2\lambda\beta}}\frac{\sum^{T-1}_{t=0}\mathbb{E}[\Vert \nabla J(\theta_t) \Vert ^2]}{T}\nn\\
    &=\mathcal{O}\Bigg(\frac{1}{T^{1-b}}+\frac{\log T}{T^b}+\frac{1}{T^{2a-2b}} \frac{\sum^{T-1}_{t=0}\mathbb{E}[\Vert \nabla J(\theta_t) \Vert ^2]}{T} \Bigg),\nn
\end{align}
where $Q_T=\frac{\Vert z_0\Vert ^2}{1-e^{-2\lambda\beta}}+\frac{\left(4Rc_{f_2}\beta+2b_{g_2}\beta+4Rb_{\eta}\alpha \right)}{\left(1-e^{-2\lambda\beta}\right)^2}
+\frac{\tau_{\beta}+1}{1-e^{-2\lambda\beta}} \left(4Rc_{f_2}\beta+2b_{g_2}\beta+4Rb_{\eta}\alpha  +c_{z}\beta^2  \right)+c_z\beta^2  +\frac{T}{1-e^{-2\lambda\beta}} \big(2\beta\left(4Rc_{f_2}\beta+b_{f_2}\beta\tau_{\beta}\right)+ 2\beta\left(b_{g_2}\beta+b'_{g_2}\beta\tau_{\beta}\right)+2\alpha\left(4Rb_{\eta}\beta+b'_{\eta}\beta\tau_{\beta}\right)\big)$, and  $b_{g_2}, b'_{g_2}, b_{\eta}, b'_{\eta} \text{ and } c_z$ are some constants defined in Lemmas \ref{lemma:eta} and \ref{lemma:8}. 
\end{Lemma}

Now we have the bounds on the stochastic bias and the  tracking error. 
%
%
From \eqref{eq:main}, we first have that
\begin{align}\label{eq:80}
    &\frac{\sum^{T-1}_{t=0}\alpha\mathbb{E}[\Vert \nabla J(\theta_t)\Vert ^2]}{2T\alpha}\nn\\
    &\leq \frac{1}{T\alpha} \Bigg( J(\theta_0)-J^*+\gamma\alpha(1+\vert  \mca\vert  Rk_1)\sqrt{\sum^{T-1}_{t=0}\mathbb{E}[\Vert \nabla J(\theta_t) \Vert ^2]}\nn\\
    &\quad\cdot\sqrt{\sum^{T-1}_{t=0}\mathbb{E}[\Vert z_t\Vert ^2]}+\sum^{T-1}_{t=0}\alpha \mathbb{E}[\zeta(\theta_t,O_t)]\nn\\
    &\quad+\sum^{T-1}_{t=0}K\alpha^2\left(r_{\max} +( R+ \gamma R(2+\vert  \mca\vert  Rk_1) \right)^2  \Bigg),
\end{align}
where $J^*=\min_{\theta} J(\theta)$  is positive and finite, and the inequality is from $\Vert G_{t+1}(\theta,\omega)\Vert \leq r_{\max}+\gamma R+R+\gamma R(1+\vert  \mca\vert  Rk_1)$. From Lemma \ref{thm:zeta}, it follows that
$
\sum_{t=0}^{T-1} \alpha\mathbb{E}[\zeta(\theta_t,O_t)]
\leq \sum^{\tau_{\alpha}}_{t=0}\alpha k_{\zeta} +\sum^{T-1}_{t=\tau_{\alpha}+1} (k_{\zeta}\alpha^2+c_{\zeta}(c_{f_1}+c_{g_1})\tau_{\alpha}\alpha^2).
$
Hence, we have that
\begin{small}
 \begin{align}
     &\frac{\sum^{T-1}_{t=0}\mathbb{E}[\Vert \nabla J(\theta_t)\Vert ^2]}{2T}\nn\\
     &\leq \Omega+\gamma(1+\vert  \mca\vert  Rk_1)\sqrt{\frac{\sum^{T-1}_{t=0}\mathbb{E}[\Vert \nabla J(\theta_t)\Vert ^2]}{T}}\sqrt{\frac{\sum^{T-1}_{t=0}\mathbb{E}[\Vert z_t\Vert ^2]}{T}},\nn
\end{align}
\end{small}
where  $\Omega\triangleq k_{\zeta}\frac{\tau_{\alpha}+1}{T}+c_{\zeta}(c_{f1}+c_{g1})\tau_{\alpha}\alpha+k_{\zeta}{\alpha}+\frac{J(\theta_0)-J^*}{T\alpha}+K\alpha\left(r_{\max}+\gamma R+R+\gamma R(1+\vert  \mca\vert  Rk_1) \right)^2$.
We then plug in the tracking error in Lemma \ref{lemma:tracking}:
\begin{small}
\begin{align}
    &\frac{\sum^{T-1}_{t=0}\mathbb{E}[\Vert \nabla J(\theta_t)\Vert ^2]}{2T}
    \overset{(a)}{\leq} \Omega +\gamma(1+\vert  \mca\vert  Rk_1)\nn\\
     &\quad\cdot\sqrt{\frac{\sum^{T-1}_{t=0}\mathbb{E}[\Vert \nabla J(\theta_t)\Vert ^2]}{T}}\Bigg( \sqrt{\frac{2Q_T}{T}+\frac{32}{1-e^{-2\lambda\beta}}\frac{ \Vert R_2\Vert ^2}{\lambda\beta}}\nn\\
     &\quad+\bigg(\frac{8\alpha^2}{\beta}\frac{1}{\lambda^3}(1+\gamma+\gamma k_1R\vert  \mca\vert  )^2\frac{1}{1-e^{-2\lambda\beta}}\nn\\
     &\quad\cdot\frac{\sum^{T-1}_{t=0}\mathbb{E}[\Vert \nabla J(\theta_t) \Vert ^2]}{T}\bigg)^{0.5}
     \Bigg)\nn\\
     &=\Omega +\gamma(1+\vert  \mca\vert  Rk_1)\sqrt{\frac{\sum^{T-1}_{t=0}\mathbb{E}[\Vert \nabla J(\theta_t)\Vert ^2]}{T}}\nn\\
     &\quad\cdot\sqrt{\frac{2Q_T}{T}+\frac{32}{1-e^{-2\lambda\beta}}\frac{ \Vert R_2\Vert ^2}{\lambda\beta}}\nn\\
    &\quad +\gamma(1+\vert  \mca\vert  Rk_1)\sqrt{\frac{8\alpha^2}{\beta}\frac{1}{\lambda^3}(1+\gamma+\gamma k_1R\vert  \mca\vert  )^2\frac{1}{1-e^{-2\lambda\beta}}}\nn\\
    &\quad\cdot\frac{\sum^{T-1}_{t=0}\mathbb{E}[\Vert \nabla J(\theta_t) \Vert ^2]}{T},
\end{align}
\end{small}
where $(a)$ is from $\sqrt{x+y}\leq \sqrt{x}+\sqrt{y}$ for any $x,y \geq 0$.  
Rearranging the terms, and choosing $\alpha$ and $\beta$ such that $\gamma(1+\vert  \mca\vert  Rk_1)\sqrt{\frac{8\alpha^2}{\beta(1-e^{-2\lambda\beta})}\frac{1}{\lambda^3}(1+\gamma+\gamma \vert  \mca\vert  Rk_1)^2}<\frac{1}{4}$, then 
\begin{align*}
    &\frac{\sum^{T-1}_{t=0}\mathbb{E}[\Vert \nabla J(\theta_t)\Vert ^2]}{T}
     \leq  
    U+V\sqrt{\frac{\sum^{T-1}_{t=0}\mathbb{E}[\Vert \nabla J(\theta_t)\Vert ^2]}{T}},
\end{align*}
where  $V=4\gamma(1+\vert  \mca\vert  Rk_1)\left(\sqrt{\frac{2Q_T}{T}+\frac{32}{1-e^{-2\lambda\beta}}\frac{\Vert R_2\Vert ^2}{\lambda\beta}}\right)$ and $U=4\Omega$. 
Hence, we have that
\begin{align}\label{eq:markovianbound}
    &\frac{\sum^{T-1}_{t=0}\mathbb{E}[\Vert \nabla J(\theta_t)\Vert ^2]}{T}\leq \left(\frac{V+\sqrt{V^2+4U}}{2}\right)^2\nn\\
    &\overset{(a)}{\leq} V^2+2U\nn\\
    &\leq 16\gamma^2(1+\vert  \mca\vert  Rk_1)^2\Bigg({\frac{2Q_T}{T}}+{\frac{32}{1-e^{-2\lambda\beta}}\frac{\Vert R_2\Vert ^2}{\lambda\beta}}\Bigg)+8\Omega\nn\\
     &=\mathcal{O}\left(\frac{1}{T^{1-a}}+\frac{\log T}{T^a}+\frac{1}{T^{1-b}} +\frac{\log T}{T^b}\right),
\end{align}
where $(a)$ is from $(x+y)^2\leq 2x^2+2y^2$ for any $x,y \geq 0$, and the last step is due to the fact that $\alpha=\mathcal{O}\left( T^{-a} \right)$, $\beta=\mathcal{O}\left( T^{-b}\right)$, $1-e^{-2\lambda\beta}=\mathcal{O}\left( T^{-b}\right)$, $\frac{Q_T}{T}=\mathcal{O}\left( \frac{1}{T^{1-b}}+\frac{\log T}{T^b}\right)$, $\frac{\Vert R_2\Vert ^2}{\beta}=\mathcal{O}\left(\frac{\alpha^4}{\beta^2}\right)=\mathcal{O}(T^{-2a})$ which is from $a\geq b \geq 0$. 
This completes the proof of Theorem \ref{thm:main}.

\subsection{Proof of Lemma \ref{lemma:tracking}}\label{app:lemmas}
Recall that $z_t=\omega_t-\omega^*(\theta_t)$, then 
\begin{align}  \label{eq:zupdate}
             z_{t+1}&=z_t+\beta(f_2(\theta_t,O_t)+g_2(\theta_t,O_t))+\omega^*(\theta_t)-\omega^*(\theta_{t+1}),  \nn\\
             \theta_{t+1}&=\theta_t+\alpha(f_1(\theta_t,O_t)+g_1(\theta_t,z_t,O_t)),  
\end{align}  
where 
$f_1(\theta_t, O_t) \triangleq  \delta_{t+1}(\theta_t)\phi_t-\gamma\phi_t^\top  \omega^*(\theta_t)\hat{\phi}_{t+1}(\theta_t), $
$g_1(\theta_t, z_t, O_t)  \triangleq  -\gamma\phi_t^\top  z_t\hat{\phi}_{t+1}(\theta_t), $
$f_2(\theta_t,O_t) \triangleq (\delta_{t+1}(\theta_t)-\phi_t^\top  \omega^*(\theta_t))\phi_t,$ and 
$g_2(z_t,O_t)  \triangleq  -\phi_t^\top  z_t\phi_t.$
We then develop upper bounds on functions $f_1,g_1,f_2,g_2$ as follows.
\begin{Lemma}\label{Lemma:3}
	For $\Vert \theta\Vert \leq R$, $\Vert z\Vert \leq 2R$,  $\Vert f_1(\theta,O_t)\Vert \leq c_{f_1},$ $\Vert g_1(\theta,z,O_t)\Vert \leq c_{g_1},$ $\vert  f_2(\theta,O_t)\vert  \leq c_{f_2}$ and $\vert  g_2(\theta,O_t)\vert  \leq c_{g_2}$,
	where $c_{f_2}=r_{\max}+(1+\gamma)R+\frac{1}{\lambda}(r_{\max}+(1 + \gamma)R)$, and $c_{g_2}=2R$.
\end{Lemma}
\begin{proof}
	This lemma follows  from \eqref{eq:30} \eqref{eq:39} and \eqref{eq:w*lip}.
\end{proof}

We then decompose the tracking error as follows
\begin{align}\label{z_1}
    &\vert  \vert  z_{t+1}\vert  \vert  ^2=\vert  \vert  z_t\vert  \vert  ^2+2\beta \langle z_t, f_2(\theta_t,O_t)\rangle +2\beta\langle z_t,g_2(z_t,O_t)\rangle \nn\\
    &\quad+2\langle z_t, \omega^*(\theta_t)-\omega^*(\theta_{t+1})\rangle\nn\\ 
    &\quad+\vert  \vert  \beta f_2(\theta_t,O_t)+\beta g_2(z_t,O_t)+\omega^*(\theta_t)-\omega^*(\theta_{t+1})\vert  \vert  ^2\nonumber\\
    &\leq\vert  \vert  z_t\vert  \vert  ^2+2\beta\langle z_t, f_2(\theta_t,O_t)\rangle +2\beta\langle z_t,\Bar{g}_2(z_t)\rangle\nn\\
    &\quad+2\langle z_t, \omega^*(\theta_t)-\omega^*(\theta_{t+1})\rangle +2\beta\langle z_t,g_2(z_t,O_t)-\Bar{g}_2(z_t)\rangle \nn\\
    &\quad+3\beta^2c_{f_2}^2+3\beta^2c_{g_2}^2\nn\\
    &\quad+ {6}(1+\gamma+\gamma R\vert  \mca\vert  k_1)^2\alpha^2 (c_{f_1}^2+c_{g_1}^2)/{\lambda^2},
\end{align}
where $\Bar{g}_2(z)\triangleq -Cz$, and the inequality follows from Lemma \ref{Lemma:3} and  Lemma \ref{Lemma:Lipschitz}. 

Define $\zeta_{f_2}(\theta,z,O_t)\triangleq\langle z, f_2(\theta,O_t) \rangle $, and $\zeta_{g_2}(z,O_t)\triangleq\langle z,g_2(z,O_t)-\Bar{g}_2(z)\rangle$.
We then characterize the bounds on and the Lipschitz smoothness of $\zeta_{f_2}$ and $\zeta_{g_2}$.
\begin{Lemma}\label{Lemma:5}
	For any $\theta,\theta_1,\theta_2 \in\{\theta:\Vert \theta\Vert \leq R\}$ and any $z,z_1,z_2\in\{z:\Vert z\Vert \leq 2R\}$,
	1) $\vert  \zeta_{f_2}(\theta,z,O_t) \vert   \leq 2Rc_{f_2}$;
	2) $\vert  \zeta_{f_2}(\theta_1,z_1,O_t)-\zeta_{f_2}(\theta_2,z_2,O_t) \vert   \leq k_{f_2}\Vert \theta_1-\theta_2 \Vert +c_{f_2}\Vert z_1-z_2\Vert $, where $k_{f_2}=2R(1+\gamma+\gamma Rk_1\vert  \mca\vert  )(1+\frac{1}{\lambda})$;
	3) $\vert  \zeta_{g_2}(z,O_t) \vert   \leq 8R^2 $;
	and
	4) $\vert  \zeta_{g_2}(z_1,O_t)-\zeta_{g_2}(z_2,O_t) \vert   \leq 8R\Vert z_1-z_2\Vert $.
\end{Lemma}
\begin{proof}
	 1) and 3) follow directly from the definition and Lemma \ref{Lemma:3}.
	For 2), it can be shown that
	\begin{align}
		&\vert  \zeta_{f_2}(\theta_1,z_1,O_t)-\zeta_{f_2}(\theta_2,z_2,O_t)\vert  \nn\\
		&\leq \vert  \langle z_1, f_2(\theta_1,O_t) \rangle-\langle z_1, f_2(\theta_2,O_t)\vert  \nn\\
		&\quad+\vert  \langle z_1, f_2(\theta_2,O_t)-\langle z_2, f_2(\theta_2,O_t) \rangle\vert  \nn\\
		&\leq 2R \Vert  f_2(\theta_1,O_t)-f_2(\theta_2,O_t)\Vert +\Vert f_2(\theta_2,O_t) \Vert  \Vert z_1-z_2 \Vert \nn\\
		&\leq 2R(\vert  \delta_{t+1}(\theta_1)-\delta_{t+1}(\theta_2)\vert  +\Vert \omega^*(\theta_1)-\omega^*(\theta_2) \Vert )\nn\\
		&\quad+c_{f_2}\Vert z_1-z_2 \Vert \nn\\
		& {\leq} k_{f_2}\Vert \theta_1-\theta_2\Vert +c_{f_2}\Vert z_1-z_2\Vert ,
	\end{align}
	where the last inequality is from the fact that both $\delta(\theta)$ and $\omega^*(\theta)$ are Lipschitz.
	
	
	To prove 4), we have that 
	\begin{align}
		&\vert  \zeta_{g_2}(z_1,O_t)-\zeta_{g_2}(z_2,O_t)\vert  \nn\\
		&=\vert  \langle z_1, -\phi_t^\top z_1\phi_t+\mathbb{E}[\phi_t^\top z_1\phi_t]\rangle\nn\\
		&\quad-\langle z_1, -\phi_t^\top z_2\phi_t+\mathbb{E}[\phi_t^\top z_2\phi_t]\rangle+\langle z_1, -\phi_t^\top z_2\phi_t\nn\\
		&\quad+\mathbb{E}[\phi_t^\top z_2\phi_t]\rangle-\langle z_2, -\phi_t^\top z_2\phi_t+\mathbb{E}[\phi_t^\top z_2\phi_t]\rangle\vert  \nn\\
		&\leq 8R\Vert z_1-z_2\Vert .
	\end{align}
\end{proof}

Now we are ready to bound the tracking error. Note that $\langle z_t,\Bar{g}_2(z_t)\rangle=-z_t^\top C z_t$, then \eqref{z_1} can be bounded as follows
\begin{align}\label{eq:69}
    &\vert  \vert  z_{t+1}\vert  \vert  ^2\leq (1-2\beta\lambda)\Vert z_t\Vert ^2+2\beta\zeta_{f_2}(\theta_t,z_t,O_t)\nn\\
    &\quad+2\beta\zeta_{g_2}(z_t,O_t)+2\langle z_t,\omega^*(\theta_t)-\omega^*(\theta_{t+1})\rangle+3\beta^2c_{f_2}^2\nonumber\\
    &\quad+3\beta^2c_{g_2}^2+ \frac{6}{\lambda^2}(1+\gamma+\gamma R\vert  \mca\vert  k_1)^2\alpha^2 (c_{f_1}^2+c_{g_1}^2).
\end{align}
Taking expectation on both sides of \eqref{eq:69},  applying it recursively and using the fact that $1-2\beta\lambda \leq e^{-2\beta\lambda}$, we obtain 
%
\begin{align}\label{eq:tracking}
     \mathbb{E}[\vert  \vert  z_{t+1}\vert  \vert  ^2&\leq A_t \vert  \vert  z_0\vert  \vert  ^2+2\sum_{i=0}^t B_{it}\nn\\
    & \quad+2\sum_{i=0}^t C_{it}+2\sum_{i=0}^t D_{it}+c_z\sum_{i=0}^t E_{it},
\end{align}
 where 
\begin{align}
A_t&=e^{-2\lambda \sum_{i=0}^t  \beta}, \nn\\
B_{it}&=e^{-2\lambda\sum_{k=i+1}^t \beta} \beta\mathbb{E}[\zeta_{f_2}(z_i,\theta_i,O_i)], \nn\\
C_{it}&=e^{-2\lambda\sum_{k=i+1}^t  \beta} \beta\mathbb{E}[\zeta_{g_2}(z_i,O_i)],\nn\\
D_{it}&=e^{-2\lambda\sum_{k=i+1}^t  \beta} \mathbb{E}[\langle z_i,\omega^*(\theta_i)-\omega^*(\theta_{i+1})\rangle],\nn\\
E_{it}&=e^{-2\lambda\sum_{k=i+1}^t  \beta} \beta^2,
\end{align}
and $c_z=3\left(c_{f_2}^2+c_{g_2}^2+\frac{2}{\lambda^2}(1+\gamma+\gamma R\vert  \mca\vert  k_1)^2(c_{f_1}^2+c_{g_1}^2)\right) $.

To bound \eqref{eq:tracking}, we provide the following lemmas.

\begin{Lemma}\label{lemma:6}
	Define $\tau_{\beta}=\min \left\{ k: m\rho^k \leq \beta \right\}$.
	If $t\leq \tau_{\beta}$, then
$
		\mathbb{E}[\zeta_{f_2}(\theta_t,z_t,O_t)]\leq 2Rc_{f_2};
$
	and if $t> \tau_{\beta}$, then
$
		\mathbb{E}[\zeta_{f_2}(\theta_t,z_t,O_t)]\leq 4Rc_{f_2}\beta+b_{f_2}\tau_{\beta}\beta,
$
	where $b_{f_2}=( c_{f_2}(c_{f_2}+c_{g_2})+ (k_{f_2}(c_{f_1}+c_{g_1})+c_{f_2}\frac{1}{\lambda}(1+\gamma+\gamma R\vert  \mca\vert  k_1)(c_{f_1}+c_{g_1})))$.
\end{Lemma}
\begin{proof}
	We first note that 
	\begin{align*}
		&\Vert z_{t+1}-z_t\Vert \nn\\
		&=\Vert \beta(f_2(\theta_t,O_t)+g_2(z_t,O_t))+\omega^*(\theta_t)-\omega^*(\theta_{t+1}) \Vert \nn\\
		&\leq (c_{f_2}+c_{g_2})\beta+\frac{1}{\lambda}(1+\gamma+\gamma R\vert  \mca\vert  k_1)(c_{f_1}+c_{g_1})\alpha,
	\end{align*}
	where the last step is due to \eqref{eq:w*lip}.
	Furthermore, due to part 2) in Lemma \ref{Lemma:5}, $\zeta_{f_2}$ is Lipschitz in both $\theta$ and $z$, then we have that for any $\tau\leq t$
	\begin{align}\label{eq:64}
		&\vert  \zeta_{f_2}(\theta_t,z_t,O_t)-\zeta_{f_2}(\theta_{t-\tau},z_{t-\tau},O_t)\vert  \nn\\
		&\overset{(a)}{\leq}
%
		c_{f_2}(c_{f_2}+c_{g_2})\sum^{t-1}_{i=t-\tau}\beta+\bigg(k_{f_2}(c_{f_1}+c_{g_1})\nn\\
		&\quad+c_{f_2}\frac{1}{\lambda}(1+\gamma+\gamma R\vert  \mca\vert  k_1)(c_{f_1}+c_{g_1})\bigg)\sum^{t-1}_{i=t-\tau}\alpha,
	\end{align}
	where in $(a)$, we apply \eqref{eq:w*lip} and Lemma \ref{Lemma:3}.
	
	Define an independent random variable $\hat O=(\hat S,\hat A,\hat R,\hat S')$, where $(\hat S,\hat A)\sim \mu $, $\hat S'\sim\mathsf P(\cdot\vert  \hat S,\hat A)$ is the subsequent state, and $\hat R$ is the reward. Then it can be shown that
	\begin{align}
		&\mathbb{E}[\zeta_{f_2}(\theta_{t-\tau},z_{t-\tau},O_t)] \nn\\
		&\overset{(a)}{\leq} \vert  \mathbb{E}[\zeta_{f_2}(\theta_{t-\tau},z_{t-\tau},O_t)]-\mathbb{E}[\zeta_{f_2}(\theta_{t-\tau},z_{t-\tau},\hat O)]\vert  \nn\\
		&\leq 4Rc_{f_2}m\rho^{\tau},
	\end{align}
	where (a) is due to the fact that $\mathbb{E}[\zeta_{f_2}(\theta_{t-\tau},z_{t-\tau},\hat O)]=0$, and the last inequality follows from Assumption \ref{ass:1}.
	
	If $t\leq \tau_{\beta}$, the result follows due to  $\vert  \zeta_{f_2}(\theta,z_t,O_t)\vert  \leq 2Rc_{f_2}$.
%
 
	If $t> \tau_{\beta}$, we choose $\tau=\tau_{\beta}$ in \eqref{eq:64}. Then, 
	\begin{align*}
		&\mathbb{E}[\zeta_{f_2}(\theta_t,z_t,O_t)]\leq  \mathbb{E}[\zeta_{f_2}(\theta_{t-\tau_{\beta}},z_{t-\tau_{\beta}},O_t)]\nn\\
		&\quad+\bigg(c_{f_2}\frac{1}{\lambda}(1+\gamma+\gamma R\vert  \mca\vert  k_1)(c_{f_1}+c_{g_1})+k_{f_2}(c_{f_1}+c_{g_1})\bigg)\nn\\
	&	\quad \cdot\sum^{t-1}_{i=t-\tau_{\beta}}\alpha +c_{f_2}(c_{f_2}+c_{g_2})\sum^{t-1}_{i=t-\tau_{\beta}}\beta \nn\\
		&\leq 4Rc_{f_2}\beta+\bigg( c_{f_2}(c_{f_2}+c_{g_2})+ \bigg(k_{f_2}(c_{f_1}+c_{g_1})\nn\\
		&\quad+c_{f_2}\frac{1}{\lambda}(1+\gamma+\gamma R\vert  \mca\vert  k_1)(c_{f_1}+c_{g_1})\bigg)\bigg)\tau_{\beta}\beta,
	\end{align*}
	where in the last step we upper bound $\alpha$ using $\beta$. Note that this will not change the order of the bound.
\end{proof}

Define the following constants:\begin{small}
\begin{flalign}
b_{\eta}&=(1+\gamma+\gamma k_1R\vert  \mca\vert  )\left(\frac{1+\lambda+2\gamma (1+k_1R\vert  \mca\vert  )}{\lambda^2}\right)(r_{\max}+2R),\nn\\
b'_{\eta}&= k'_{\eta}\left(c_{f_2}+c_{g_2}\right)+ \big(k_{\eta}+\frac{k'_{\eta}}{\lambda}\left(1+\gamma+\gamma R\vert  \mca\vert  k_1\right)\big) \left(c_{f_1}+c_{g_1}\right),\nn\\
k_{\eta}&=2R\bigg( \frac{1}{\lambda}(1+\gamma+\gamma k_1R\vert  \mca\vert  )\left( k_3+\frac{K}{2} \right)+(r_{\max}+\gamma R+R)\nn\\ &\quad\cdot\left(1+\lambda+2\gamma (1+k_1R\vert  \mca\vert  )\right)\frac{2}{\lambda^2}(\gamma \vert  \mca\vert  ( k_1+k_2R))\bigg),\nn\\
k'_{\eta}&= \left( \frac{1+\lambda+2\gamma (1+k_1R\vert  \mca\vert  )}{\lambda^2(r_{\max}+\gamma R+R)^{-1}}\right)(1+\gamma+\gamma k_1R\vert  \mca\vert  ).\nn
\end{flalign}\end{small}
\begin{Lemma}\label{lemma:eta}
Let $\eta(\theta,z,O_t)=\langle z, -\nabla \omega^*(\theta)^\top(G_{t+1}(\theta,\omega^*(\theta))\\+{\nabla J(\theta)}/{2})\rangle$, then if $t\leq \tau_{\beta}$, $\mathbb{E}[\eta(\theta_t,z_t,O_t)]\leq 2Rb_{\eta}$;
and if $t> \tau_{\beta}$, then $\mathbb{E}[\eta(\theta_t,z_t,O_t)]\leq 4Rb_{\eta}\beta+b'_{\eta}\tau_{\beta}\beta$.
\end{Lemma}
\begin{proof}
From the update of $z_t$ in \eqref{eq:zupdate}, we first have
\begin{align}
    &\Vert z_{t+1}-z_t\Vert \nn\\
    &=\Vert \beta(f_2(\theta_t,O_t)+g_2(z_t,O_t))+\omega^*(\theta_t)-\omega^*(\theta_{t+1}) \Vert \nn\\
    &\leq (c_{f_2}+c_{g_2})\beta+\frac{1}{\lambda}(1+\gamma+\gamma R\vert  \mca\vert  k_1)(c_{f_1}+c_{g_1})\alpha,
\end{align}
where the last step is due to the fact that $\Vert f_2(\theta,O_t)\Vert \leq c_{f_2}$, $\Vert g_2(\theta,O_t)\Vert \leq c_{g_2}$ and $\omega^*(\theta)$ is Lipschitz in $\theta$ (Lemma \ref{Lemma:Lipschitz}). 

Recall that both ${\nabla J(\theta)}/{2}$, and $G_{t+1}(\theta,\omega^*(\theta))$ are Lipschitz in $\theta$ from  \eqref{lemma:Lsmooth} and \eqref{eq:G^*lip}. Also note that $\nabla \omega^*(\theta)=C^{-1} \nabla \mathbb{E}[\delta_{S,A,S'}(\theta)\phi_{S,A}]$, which implies that $\Vert \nabla \omega^*(\theta)\Vert ^2\leq \frac{1}{\lambda^2}(1+\gamma +\gamma k_1R\vert  \mca\vert  )^2$. Then $\nabla \omega^*(\theta) $ is Lipschitz in $\theta$:
\begin{small}
\begin{align}\label{eq:nablaw^*lip}
    &\Vert \nabla \omega^*(\theta_1)-\nabla \omega^*(\theta_2) \Vert \nn\\
    &\leq \Vert C^{-1}\Vert  \big \Vert \nabla \left (\mathbb{E}_{\mu}\left [\delta_{S,A,S'}\left (\theta_1\right )\phi_{S,A}\right ]\right )-\nabla \left (\mathbb{E}_{\mu}\left [\delta_{S,A,S'}\left (\theta_2\right )\phi_{S,A}\right ]\right )\big\Vert \nn\\
    &\leq \frac{\gamma}{\lambda}\Bigg\Vert \mathbb{E}_{\mu}\Bigg [ \Bigg(\sum_{a'\in\mathcal{A}} \bigg(\nabla\pi_{\theta_1}\left (a'\vert  S'\right )\theta_1^\top \phi_{S',a'}-\nabla\pi_{\theta_2}\left (a'\vert  S'\right )\theta_2^\top \phi_{S',a'}\nn\\
    &\quad+\pi_{\theta_1}\left (a'\vert  S'\right )\phi_{S',a'}-\pi_{\theta_2}\left (a'\vert  S'\right )\phi_{S',a'}\bigg)\Bigg)\phi_{S,A}^\top\Bigg ]\Bigg\Vert \nn\\
    &\leq \frac{\gamma}{\lambda}\Bigg\Vert \mathbb{E}_{\mu}\Bigg [\Bigg( \sum_{a'\in\mathcal{A}} (\nabla\pi_{\theta_1}\left (a'\vert  S'\right )\theta_1^\top \phi_{S',a'}-\nabla\pi_{\theta_2}\left (a'\vert  S'\right )\theta_1^\top \phi_{S',a'}\nn\\
    &\quad+\nabla\pi_{\theta_2}\left (a'\vert  S'\right )\theta_1^\top \phi_{S',a'}-\nabla\pi_{\theta_2}\left (a'\vert  S'\right)\theta_2^\top \phi_{S',a'})\Bigg)\phi_{S,A}^\top\Bigg ]\Bigg\Vert \nn\\
    &\quad+\frac{\gamma}{\lambda}\Bigg\Vert \mathbb{E}_{\mu}\Bigg [\Big (\sum_{a'\in \mca}   (\pi_{\theta_1} (a'\vert  S' )-\pi_{\theta_2} (a'\vert  S' ))\phi_{S',a'}\Big )\phi_{S,A}^\top\Bigg ]\Bigg\Vert \nn\\
    &\leq  \frac{2\gamma}{\lambda} \vert  \mca\vert  (Rk_2+k_1)\Vert \theta_1-\theta_2\Vert .
\end{align}
\end{small}

Therefore, 
\begin{small}
\begin{align}\label{eq:etalip1}
    &\vert  \eta(\theta_1,z_1,O_t)-\eta(\theta_2,z_2,O_t)\vert  \nn\\
    &\leq 0.5\big\vert  \big\langle z_1, \nabla \omega^*(\theta_1)^\top\left(2G_{t+1}(\theta_1,\omega^*(\theta_1))+ {\nabla J(\theta_1)} \right)\big\rangle\nn\\
    &\quad-\big\langle z_2, \nabla \omega^*(\theta_1)^\top\left(2G_{t+1}(\theta_1,\omega^*(\theta_1))+ {\nabla J(\theta_1)}\right)\big\rangle\big\vert  \nn\\
    &\quad+0.5\big\vert  \big\langle z_2, \nabla \omega^*(\theta_1)^\top\left(2G_{t+1}(\theta_1,\omega^*(\theta_1))+ {\nabla J(\theta_1)}\right)\big\rangle\nn\\
    &\quad-\big\langle z_2, \nabla \omega^*(\theta_2)^\top\left(2G_{t+1}(\theta_2,\omega^*(\theta_2))+ \nabla J(\theta_2)\right)\big\rangle\big\vert  \nn\\
    &\leq \frac{1}{\lambda}\bigg(\left(1+\frac{1}{\lambda}+\frac{2\gamma (1+ k_1R\vert  \mca\vert  )}{\lambda}\right)\nn\\
    &\quad\cdot(r_{\max}+\gamma R+R)(1+\gamma +\gamma k_1R\vert  \mca\vert  )\bigg)\Vert z_1-z_2\Vert \nn\\
    &\quad+R\big\Vert  \nabla \omega^*(\theta_1)^\top\left(2G_{t+1}(\theta_1,\omega^*(\theta_1))+\nabla J(\theta_1)\right)\nn\\
    &\quad-\nabla \omega^*(\theta_2)^\top\left(2G_{t+1}(\theta_2,\omega^*(\theta_2))+\nabla J(\theta_2)\right)\big\Vert .
\end{align}
\end{small}
Consider the last term in \eqref{eq:etalip1}. We know that $\nabla \omega^*(\theta ) $ and $ G_{t+1}(\theta,\omega^*(\theta))+\frac{\nabla J(\theta)}{2} $ are both Lipschitz in $\theta$  from \eqref{lemma:Lsmooth}, \eqref{eq:G^*lip} and \eqref{eq:nablaw^*lip}. It can then be shown that $ \nabla\omega^*(\theta)\left( G_{t+1}(\theta,\omega^*(\theta))+\frac{\nabla J(\theta)}{2}\right)$ is also Lipschitz with constant 
$\frac{1}{\lambda}(1+\gamma+\gamma k_1R\vert  \mca\vert  )\left( k_3+\frac{K}{2} \right) +\left(1+\frac{2\gamma (1+k_1R\vert  \mca\vert  )}{\lambda}+\frac{1}{\lambda}\right) (r_{\max}+\gamma R+R)\frac{2}{\lambda}(\gamma \vert  \mca\vert  ( k_1+k_2R))$.
Plugging this into \eqref{eq:etalip1}, we obtain that
\begin{align*}
    &\vert  \eta(\theta_1,z_1,O_t)-\eta(\theta_2,z_2,O_t)\vert   
     \leq k_{\eta}\Vert  \theta_1-\theta_2\Vert +k'_{\eta}\Vert z_1-z_2\Vert .
\end{align*}
Then for any $\tau\geq 0$,
\begin{align}\label{eq:88}
    &\vert  \eta\left(\theta_t,z_t,O_t\right)-\eta\left(\theta_{t-\tau},z_{t-\tau},O_t\right)\vert  \nn\\
    &\overset{}{\leq} 
    k'_{\eta}\left(c_{f_2}+c_{g_2}\right)\sum^{t-1}_{i=t-\tau}\beta+\bigg(k_{\eta}\left(c_{f_1}+c_{g_1}\right)\nn\\
    &\quad+k'_{\eta}\frac{1}{\lambda}\left(1+\gamma+\gamma R\vert  \mca\vert  k_1\right)\left(c_{f_1}+c_{g_1}\right)\bigg)\sum^{t-1}_{i=t-\tau}\alpha.
\end{align}

Define an independent random variable $\hat O=(\hat S,\hat A,\hat R,\hat S')$, where $(\hat S,\hat A)\sim \mu $, $\hat S'\sim\mathsf P(\cdot\vert  \hat S,\hat A)$ is the subsequent state, and $\hat R$ is the reward. Then it can be shown that
\begin{align}\label{eq:etabound}
    &\mathbb{E}[\eta(\theta_{t-\tau},z_{t-\tau},O_t)] \nn\\
    &\overset{(a)}{\leq} \vert  \mathbb{E}[\eta(\theta_{t-\tau},z_{t-\tau},O_t)]-\mathbb{E}[\eta(\theta_{t-\tau},z_{t-\tau},\hat O)]\vert  \nn\\
    &\leq 4Rb_{\eta}m\rho^{\tau},
\end{align}
where (a) is due to the fact that $\mathbb{E}[\eta(\theta_{t-\tau},z_{t-\tau},\hat O)]=0$, and $b_{\eta}\triangleq\sup_{\Vert \theta\Vert \leq R} \left\Vert  \nabla \omega^*(\theta)^\top\left(G_{t+1}(\theta,\omega^*(\theta))+{\nabla J(\theta)}/{2}\right)\right\Vert = {1}/{\lambda} (1+\gamma+\gamma k_1R\vert  \mca\vert  )\left(1+{1}/{\lambda}+{2\gamma (1+k_1R\vert  \mca\vert  )}/{\lambda}\right)(r_{\max}+(1+\gamma)R)$.

If $t\leq \tau_{\beta}$, the conclusion is straightforward by noting that $\vert  \eta(\theta,z,O_t)\vert  \leq 2Rb_{\eta}$ for any $\Vert \theta\Vert \leq R$ and $\Vert z\Vert \leq 2R$. If $t> \tau_{\beta}$, we choose $\tau=\tau_{\beta}$ in \eqref{eq:88} and \eqref{eq:etabound}. Then, it can be shown that 
\begin{align}
    &\mathbb{E}[\eta\left(\theta_t,z_t,O_t\right)]\nn\\
    &\leq  \mathbb{E}[\eta\left(\theta_{t-\tau_{\beta}},z_{t-\tau_{\beta}},O_t\right)]+k'_{\eta}\left(c_{f_2}+c_{g_2}\right)\hspace{-0.1cm}\sum^{t-1}_{i=t-\tau_{\beta}}\hspace{-0.1cm}\beta+\hspace{-0.1cm}\sum^{t-1}_{i=t-\tau_{\beta}}\hspace{-0.1cm}\alpha\nn\\
    &\quad\cdot\left(k_{\eta}\left(c_{f_1}+c_{g_1}\right)+k'_{\eta}\frac{1}{\lambda}\left(1+\gamma+\gamma R\vert  \mca\vert  k_1\right)\left(c_{f_1}+c_{g_1}\right)\right)\nn\\
    &\leq 
    4Rb_{\eta}\beta+b'_{\eta}\beta\tau_{\beta}.
\end{align}
\end{proof}

The next lemma provides a bound on $\mathbb{E}[\zeta_{g_2}(z_t,O_t)]$. 
\begin{Lemma}\label{lemma:8}
If $t\leq \tau_{\beta}$, then
$\mathbb{E}[\zeta_{g_2}(z_t,O_t)] \leq b_{g_2}$;
and if $t> \tau_{\beta}$, then
$\mathbb{E}[\zeta_{g_2}(z_t,O_t)] \leq b_{g_2}\beta+b'_{g_2}\tau_{\beta}\beta$,
where $b'_{g_2}=8R(c_{f_2}+c_{g_2})+\frac{1}{\lambda}(1+\gamma+\gamma R\vert  \mca\vert  k_1)(c_{f_1}+c_{g_1})$ and
$b_{g_2}=16R^2$.
\end{Lemma}
\begin{proof}
	The proof is similar to the one for Lemma \ref{lemma:6}.
\end{proof}

Now we bound the terms in \eqref{eq:tracking} as follows.
If $t\leq\tau_{\beta}$, 
\begin{align}\label{eq:B}
    \sum^t_{i=0} B_{it}
    &\leq 2\beta Rc_{f_2}\sum^t_{i=0}e^{-2\lambda(t-i)\beta}\leq\frac{2\beta Rc_{f_2}}{1-e^{-2\lambda\beta}}.
\end{align}
If $t>\tau_{\beta}$, we have that
\begin{align}\label{eq:B>}
    &\sum^t_{i=0} B_{it}
    \leq \beta(2Rc_{f_2})\sum^{\tau_{\beta}}_{i=0}e^{-2\lambda\sum_{k=i+1}^t \beta}\nn\\
    &\quad+\sum^t_{i=\tau_{\beta}+1} e^{-2\lambda(t-i)\beta}\beta(4Rc_{f_2}\beta+b_{f_2}\beta\tau_{\beta})\nn\\
    &\leq 2Rc_{f_2}\beta \frac{e^{-2\lambda(t-\tau_{\beta})\beta}}{1-e^{-2\lambda\beta}}+\frac{\beta(4Rc_{f_2}\beta+b_{f_2}\beta\tau_{\beta})}{1-e^{-2\lambda\beta}} .
\end{align}

Similarly, using Lemma \ref{lemma:8}, we can bound the third term in \eqref{eq:tracking} as follows.
If $t\leq\tau_{\beta}$, we have that
\begin{align}\label{eq:C}
    \sum^t_{i=0}C_{it}&=\sum^t_{i=0} e^{-2\lambda\sum_{k=i+1}^t  \beta} \beta b_{g_2} \leq\frac{\beta b_{g_2} }{1-e^{-2\lambda\beta}}.
\end{align}
If $t>\tau_{\beta}$, we have that 
\begin{align}\label{eq:C>}
    \sum^t_{i=0}C_{it}&\leq b_{g_2}\beta \frac{e^{-2\lambda(t-\tau_{\beta})\beta}}{1-e^{-2\lambda\beta}}+\frac{\beta(b_{g_2}\beta+b'_{g_2}\beta\tau_{\beta})}{1-e^{-2\lambda\beta}}.
\end{align}

The last step to bound the tracking error is to bound $\sum^t_{i=0} D_{it}$, which is shown in the following lemma.
\begin{Lemma}\label{thm:D}
If $t\leq \tau_{\beta}$, $\sum^t_{i=0} D_{it}\leq P_t+ \frac{2Rb_{\eta}\alpha}{1-e^{-2\lambda\beta}}$;
and if $t>\tau_{\beta}$, 
$\sum^t_{i=0} D_{it}\leq P_t+2Rb_{\eta}\alpha\frac{e^{-2\lambda(t-\tau_{\beta})\beta}}{1-e^{-2\lambda\beta}}+\alpha(4Rb_{\eta}\beta+b'_{\eta}\beta\tau_{\beta})\frac{1}{1-e^{-2\lambda\beta}}$,
where  \begin{align}
P_t&=\sum^t_{i=0} e^{-2\lambda \left(t-i\right)\beta}  \bigg( \bigg(\frac{\lambda\beta}{8}+\frac{8\alpha^2}{\beta}\frac{1}{\lambda^3}\gamma^2\left(1+k_1R\vert  \mca\vert  \right)^2\nn\\
&\quad\cdot(1+\gamma+\gamma k_1R\vert  \mca\vert  )^2\bigg)\mathbb{E}[\Vert z_i\Vert ^2]+\frac{8\Vert R_2\Vert ^2}{\lambda\beta} \nn\\
&\quad+\frac{2\alpha^2}{\beta\lambda^3} \left(1+\gamma+\gamma k_1R\vert  \mca\vert  \right)^2\mathbb{E}\left[\Vert \nabla J\left(\theta_i\right) \Vert ^2\right]\bigg).\end{align}
\end{Lemma}

\begin{proof}
We first have that 
\begin{align}\label{eq:65}
    &\mathbb{E}\left[\left\langle z_i, w^*\left(\theta_i\right)-w^*\left(\theta_{i+1}\right)\right\rangle\right]\nn\\
    &\overset{(*)}{=}\mathbb{E}\left[\left\langle z_i, \nabla \omega^*\left(\theta_i\right)^\top\left(\theta_i-\theta_{i+1}\right)+R_2\right\rangle\right]\nn\\
    &= \alpha\mE[\eta(\theta_i,z_i,O_i)]+\frac{1}{2}\mathbb{E}\big[\big\langle z_i, -\alpha\nabla \omega^*\left(\theta_i\right)^\top\big(2G_{i+1}\left(\theta_i,\omega_i\right)\nn\\
    &\quad-2G_{i+1}\left(\theta_i,\omega^*\left(\theta_i\right)\right)-\nabla J\left(\theta_i\right)\big)+2R_2\big\rangle\big],
\end{align}
where $(*)$ follows from the Taylor expansion, and $R_2$ denotes higher order terms with $\Vert R_2\Vert =\mathcal{O}(\alpha^2)$. 

The second expectation on the RHS of \eqref{eq:65} can be bounded as follows
\begin{align}
    &\frac{1}{2}\mathbb{E}\big[\big\langle z_i, -\alpha\nabla \omega^*\left(\theta_i\right)^\top\big(2G_{i+1}\left(\theta_i,\omega_i\right)-2G_{i+1}\left(\theta_i,\omega^*\left(\theta_i\right)\right)\nn\\
    &\quad-\nabla J\left(\theta_i\right)\big)+2R_2\big\rangle\big]\nn\\
    &\overset{(a)}{\leq} \mathbb{E}\left[ \frac{\lambda\beta}{8}\Vert  z_i\Vert ^2\right]\nn\\
    &\quad+\mathbb{E}\bigg[\frac{8\alpha^2}{\beta} \frac{1}{\lambda^3} \gamma^2\left(\vert  \mca\vert  Rk_1+1\right)^2(1+\gamma+\gamma \vert  \mca\vert  Rk_1)^2\Vert z_i\Vert ^2\bigg]\nn\\
    &\quad+\mathbb{E}\left[\frac{2}{\lambda^3} \left(1+\gamma+\gamma k_1R\vert  \mca\vert  \right)^2\frac{ \alpha^2}{\beta}\Vert \nabla J\left(\theta_i\right) \Vert ^2\right]+\frac{8\Vert R_2\Vert ^2}{\lambda\beta}\nn\\
    &=\left(\frac{\lambda\beta}{8}+\frac{8\alpha^2}{\beta}\frac{1}{\lambda^3}\gamma^2\left(1+k_1R\vert  \mca\vert  \right)^2(1+\gamma+\gamma k_1R\vert  \mca\vert  )^2\right) \nn\\
    &\quad\cdot \mathbb{E}\left[\Vert z_i\Vert ^2\right]+\frac{8\Vert R_2\Vert ^2}{\lambda\beta}\nn\\
    &\quad+\frac{2\alpha^2}{\beta}\frac{1}{\lambda^3}\left(1+\gamma+\gamma k_1R\vert  \mca\vert  \right)^2\mathbb{E}\left[\Vert \nabla J\left(\theta_i\right) \Vert ^2\right],
\end{align}
where $(a)$ follows from $\langle x,y\rangle \leq \frac{\lambda\beta}{8}\Vert x\Vert ^2+\frac{2}{\lambda\beta}\Vert y\Vert ^2$ for any $x,y \in \mathbb{R}^N$,  $\Vert x+y+z\Vert ^2\leq 4\Vert x\Vert ^2+4\Vert y\Vert ^2+4\Vert z\Vert ^2$ for any $x,y,z \in \mathbb{R}^N$, and   Lemma \ref{Lemma:Lipschitz}. 

Thus, we have that
\begin{align}
    &\sum^t_{i=0} D_{it}\leq 
    P_t +\sum^{t}_{i=0} \alpha e^{-2\lambda \left(t-i\right)\beta}\mathbb{E}[\eta\left(\theta_i,z_i,O_i\right)].
\end{align}
With Lemma \ref{lemma:eta}, this concludes the proof.
\end{proof}

We then consider the tracking error $\mathbb{E}[\Vert z_{t} \Vert ^2]$ in \eqref{eq:tracking}. Combining all the bounds in \eqref{eq:B} \eqref{eq:B>} \eqref{eq:C} \eqref{eq:C>} and Lemma \ref{thm:D}, we have that if $t\leq \tau_{\beta}$,
\begin{align}\label{eq:order}
   &\mathbb{E}[\Vert  z_t\Vert ^2]\leq
 \Vert z_0\Vert ^2 e^{-2\lambda t\beta}+\Omega_1+2P_t,
\end{align}   
where $\Omega_1\triangleq \frac{1}{1-e^{-2\lambda\beta}} ( 4Rc_{f_2}\beta+2b_{g_2}\beta+4Rb_{\eta}\alpha+c_z\beta^2)$;
and if $t> \tau_{\beta}$,
\begin{align*}
    \mathbb{E}[\Vert  z_t\Vert ^2] 
&\leq \Vert z_0\Vert ^2 e^{-2\lambda t\beta}+2P_t+\Omega_2\nn\\
&\quad+\frac{e^{-2\lambda(t-\tau_{\beta})\beta}}{1-e^{-2\lambda\beta}} (4Rc_{f_2}\beta+2b_{g_2}\beta+4Rb_{\eta}\alpha),
\end{align*}
where $\Omega_2\triangleq \frac{1}{1-e^{-2\lambda\beta}} (2\beta(4Rc_{f_2}\beta+b_{f_2}\beta\tau_{\beta})+ 2\beta(b_{g_2}\beta+b'_{g_2}\beta\tau_{\beta})+2\alpha(4Rb_{\eta}\beta+b'_{\eta}\beta\tau_{\beta})+c_z\beta^2 )$.
We then bound $\sum^{T-1}_{t=0}\mathbb{E}[\Vert z_t\Vert ^2]$. 
The sum is divided into two parts  $\sum^{\tau_{\beta}}_{t=0}\mathbb{E}[\Vert z_t\Vert ^2]$ and $\sum^{T-1}_{t=\tau_{\beta}+1}\mathbb{E}[\Vert z_t\Vert ^2]$ as follows
\begin{align}\label{eq:tracking error H}
    &\sum^{T-1}_{t=0}\mathbb{E}[\Vert z_t\Vert ^2]=\sum^{\tau_{\beta}}_{t=0}\mathbb{E}[\Vert z_t\Vert ^2]+\sum^{T-1}_{t=\tau_{\beta}+1}\mathbb{E}[\Vert z_t\Vert ^2]\nn\\
   &\leq \sum^{\tau_{\beta}}_{t=0} \bigg(\Vert z_0\Vert ^2 e^{-2\lambda t\beta}+\Omega_1+2P_t\bigg)+\sum^{T-1}_{t=\tau_{\beta}+1} \bigg(\Vert z_0\Vert ^2 e^{-2\lambda t\beta}\nn\\
   &\quad+2P_t+\frac{e^{-2\lambda\left(t-\tau_{\beta}\right)\beta}}{1-e^{-2\lambda\beta}} \left(4Rc_{f_2}\beta+2b_{g_2}\beta+4Rb_{\eta}\alpha \right)+\Omega_2\bigg)\nn\\
    &\leq\sum^{T-1}_{t=0}\left(\Vert z_0\Vert ^2e^{-2\lambda t\beta}+ 2P_t\right)+(1+\tau_\beta)\Omega_1+(T-\tau_\beta)\Omega_2\nn\\
    &\quad+\sum^{T-1}_{t=\tau_{\beta}+1}\frac{e^{-2\lambda\left(t-\tau_{\beta}\right)\beta}}{1-e^{-2\lambda\beta}} \left(4Rc_{f_2}\beta+2b_{g_2}\beta+4Rb_{\eta}\alpha \right) \nn\\
     &\leq \frac{\Vert z_0\Vert ^2}{1-e^{-2\lambda\beta}}+\sum^{T-1}_{t=0} 2P_t+(1+\tau_\beta)\Omega_1+(T-\tau_\beta)\Omega_2\nn\\
    &\quad+\frac{\left(4Rc_{f_2}\beta+2b_{g_2}\beta+4Rb_{\eta}\alpha \right)}{\left(1-e^{-2\lambda\beta}\right)^2}\\
    &=\mathcal{O} \left( \frac{1}{\beta}+\tau_{\beta}+{T \beta\tau_{\beta}}+2\sum^{T-1}_{t=0} P_t \right).\nn
\end{align}
Let 
$
Q_T\triangleq\frac{\Vert z_0\Vert ^2}{1-e^{-2\lambda\beta}}+\frac{\left(4Rc_{f_2}\beta+2b_{g_2}\beta+4Rb_{\eta}\alpha \right)}{\left(1-e^{-2\lambda\beta}\right)^2}+(1+\tau_\beta)\Omega_1+(T-\tau_\beta)\Omega_2.
$
Then, 
$\sum^{T-1}_{t=0}\mathbb{E}[\Vert z_t\Vert ^2]\leq 2\sum^{T-1}_{t=0}P_t+Q_T.
$

Now we plug in the exact definition of $P_t$,\begin{small}
\begin{align}
    &\sum^{T-1}_{t=0}\mathbb{E}[\Vert z_t\Vert ^2]\leq 2\sum^{T-1}_{t=0}P_t+Q_T\nn\\
     &\leq Q_T+\frac{16T}{1-e^{-2\lambda\beta}}\frac{\Vert R_2\Vert ^2}{\lambda\beta}+ 2\sum^{T-1}_{t=0} \sum^t_{i=0}e^{-2\lambda \left(t-i\right)\beta}\mathbb{E}[\Vert z_i\Vert ^2]\nn\\
     &\quad\cdot\bigg(\frac{\lambda\beta}{8}+\frac{8\alpha^2}{\beta}\frac{1}{\lambda^3}\gamma^2\left(1+k_1R\vert  \mca\vert  \right)^2(1+\gamma+\gamma k_1R\vert  \mca\vert  )^2\bigg)\nn\\
    &\quad+\frac{4\alpha^2}{\beta\lambda^3}(1+\gamma+\gamma k_1R\vert  \mca\vert  )^2\sum^{T-1}_{t=0} \sum^t_{i=0}e^{-2\lambda \left(t-i\right)\beta}\mathbb{E}[\Vert \nabla J\left(\theta_i\right) \Vert ^2]\nn\\
    &\leq Q_T+\frac{16T}{1-e^{-2\lambda\beta}}\frac{\Vert R_2\Vert ^2}{\lambda\beta}+\frac{1}{1-e^{-2\lambda\beta}}\sum^{T-1}_{t=0}\mathbb{E}[\Vert  z_t\Vert ^2]\nn\\
    &\quad\cdot2\bigg(\frac{\lambda\beta}{8}+\frac{8\alpha^2}{\beta}\frac{1}{\lambda^3}\gamma^2\left(1+k_1R\vert  \mca\vert  \right)^2(1+\gamma+\gamma k_1R\vert  \mca\vert  )^2\bigg)\nn\\
    &\quad+\frac{4\alpha^2}{\lambda^3\beta}\frac{(1+\gamma+\gamma k_1R\vert  \mca\vert  )^2}{1-e^{-2\lambda\beta}}\sum^{T-1}_{t=0}\mathbb{E}[\Vert \nabla J\left(\theta_t\right) \Vert ^2],
\end{align}\end{small}
where the last step is from the double sum trick: for any $x_t\geq 0$ $\sum^{T-1}_{t=0}\sum^t_{i=0} e^{-2\lambda(t-i)\beta}x_i \leq \frac{1}{1-e^{-2\lambda\beta}}\sum^{T-1}_{t=0}x_t$.
Choose $\beta$ such that $\big(\frac{\lambda\beta}{8}+\frac{8\alpha^2}{\beta}\frac{1}{\lambda^3}\gamma^2\left(1+k_1R\vert  \mca\vert  \right)^2(1+\gamma+\gamma k_1R\vert  \mca\vert  )^2\big)\frac{1}{1-e^{-2\lambda\beta}}<\frac{1}{4}$. Then it follows that

\begin{align}\label{eq:trackingorder}
    &\frac{\sum^{T-1}_{t=0}\mathbb{E}[\Vert z_t\Vert ^2]}{T}\leq \frac{2Q_T}{T}+\frac{32}{1-e^{-2\lambda\beta}}\frac{ \Vert R_2\Vert ^2}{\lambda\beta}\nn\\
    &\quad+\frac{8\alpha^2}{\beta}\frac{1}{\lambda^3}\frac{(1+\gamma+\gamma k_1R\vert  \mca\vert  )^2}{1-e^{-2\lambda\beta}}\frac{\sum^{T-1}_{t=0}\mathbb{E}[\Vert \nabla J(\theta_t) \Vert ^2]}{T}\nn\\
    &=\mathcal{O}\Bigg(\frac{\log T}{T^b}+\frac{1}{T^{1-b}}+  \frac{\sum^{T-1}_{t=0}\mathbb{E}[\Vert \nabla J(\theta_t) \Vert ^2]}{T^{1+2a-2b}} \Bigg),
\end{align}
where the last step is from $1-e^{-2\lambda\beta}=\mathcal{O}(\beta)$ and $\Vert R_2\Vert ^2=\mathcal{O}(\alpha^4)$. This hence completes the proof of Lemma \ref{lemma:tracking}. 

\section{Analysis for Nested-loop Greedy-GQ}\label{section:nested}

\subsection{Proof of Theorem \ref{nestcov}}
Define 
$
    \hat{G}_t(\theta,w)=\frac{1}{M}\sum^M_{i=1} G_{(BT_c+M)t+BT_c+i}(\theta,w).
$
By the K-smoothness of $J(\theta)$, and following steps similar to those in the proof of Theorem \ref{thm:main}, we have that 
\begin{align}\label{eq:nestmaineq}
    &\frac{\alpha-K\alpha^2}{4} \Vert  \nabla J(\theta_t)\Vert ^2\leq J(\theta_t)-J(\theta_{t+1})\nn\\
    &\quad+2(\alpha+K\alpha^2)\left\Vert  \hat{G}_t(\theta_t,\omega_t)-\hat{G}_t(\theta_t,\omega^*(\theta_t))\right\Vert ^2\nn\\
    &\quad+\frac{1}{2}(\alpha+K\alpha^2)\left\Vert 2\hat{G}_t(\theta_t,\omega^*(\theta_t))+ \nabla J(\theta_t)\right\Vert ^2.
\end{align}
By the definition, we have that
\begin{align} 
    &\hat{G}_t(\theta_t,\omega_t)-\hat{G}_t(\theta_t,\omega^*(\theta_t))\nn\\
    &=\frac{1}{M} \sum^M_{i=1} \big( G_{(BT_c+M)t+BTc+i}(\theta_t,\omega_t)\nn\\
    &\quad-G_{(BT_c+M)t+BTc+i}(\theta_t,\omega^*(\theta_t))\big).
\end{align}
For any $\Vert \theta\Vert \leq R$ and any $\omega_1, \omega_2$, $\Vert  G_{(BT_c+M)t+BTc+i}(\theta,w_1)\\-G_{(BT_c+M)t+BTc+i}(\theta,w_2) \Vert \leq \gamma (1+\vert  \mathcal{A}\vert  Rk_1)\Vert w_1-w_2\Vert $. 
Hence we have that
\begin{small}\begin{align}
    &\left\Vert  \hat{G}_t(\theta_t,\omega_t)-\hat{G}_t(\theta_t,\omega^*(\theta_t))\right\Vert 
    \leq \gamma(1+\vert  \mathcal{A}\vert  Rk_1)\Vert z_t\Vert ,
\end{align}\end{small}
Thus 
$
    \Vert  \hat{G}_t(\theta_t,\omega_t)-\hat{G}_t(\theta_t,\omega^*(\theta_t))\Vert ^2\leq \gamma^2(1+\vert  \mathcal{A}\vert  Rk_1)^2\Vert z_t\Vert ^2.
$
Plugging this in \eqref{eq:nestmaineq}, we have that
\begin{align}\label{eq:111}
    &\frac{\alpha-K\alpha^2}{4} \Vert  \nabla J(\theta_t)\Vert ^2\nn\\
    &\leq J(\theta_t)-J(\theta_{t+1})+2(\alpha+K\alpha^2)\gamma^2(1+\vert  \mathcal{A}\vert  Rk_1)^2\Vert z_t\Vert ^2\nn\\
    &\quad+\frac{1}{2}(\alpha+K\alpha^2)\left\Vert 2\hat{G}_t(\theta_t,\omega^*(\theta_t))+\nabla J(\theta_t)\right\Vert ^2.
\end{align}
The following lemma provide the upper bounds on the two terms (proof in Section \ref{section:nestproof}).
\begin{Lemma}\label{prop:1}
For any $t\geq 1$,
\begin{align}
    &\mathbb{E}[\Vert z_{t}\Vert ^2] \leq 4R^2e^{\left(4\beta^2-\beta\lambda\right)(T_c-1)}+\frac{4\beta\lambda+2}{\lambda-4\beta}\frac{k_l+k_h}{B},\label{eq:102}\\
    &\mathbb{E}\left[\left\Vert  2\hat{G}_t(\theta_t,\omega^*(\theta_t))+\nabla J(\theta_t)\right\Vert ^2\right]\leq  \frac{4k_G}{M},\label{eq:103}
\end{align}
where $k_h=\frac{32R^2(1+\rho m-\rho)}{1-\rho}$, $k_G=8(r_{\max}+\gamma R+R)^2\left(1+\frac{1}{\lambda}+\frac{2\gamma}{\lambda}(1+Rk_1\vert  \mca\vert  )\right)^2(1+\rho(m-1))$ and $k_l=\frac{8(1+\lambda)^2(r_{\max}+R+\gamma R)^2(1+\rho m-\rho)}{1-\rho}$.
If we further let $T_c=\mathcal{O}\left(\log \frac{1}{\epsilon}\right)$ and $B=\mathcal{O}\left(\frac{1}{\epsilon}\right)$, then $\mathbb{E}[\Vert z_{T_c}\Vert ^2]\leq \mathcal{O}\left(\epsilon\right)$.
\end{Lemma}

Now we have the bound above, we hence plug them in \eqref{eq:111} and sum up w.r.t. $t$ from 0 to $T-1$.  Then
\begin{small}
\begin{align}
    &\frac{\alpha-K\alpha^2}{4} \frac{\sum_{t=0}^{T-1}\mathbb{E} [\Vert  \nabla J(\theta_t)\Vert ^2]}{T}\nn\\
    &\leq \frac{J(\theta_0)-J^*}{T}+2(\alpha+K\alpha^2)L^2\frac{\sum_{t=0}^{T-1}\mathbb{E}[\Vert  z_{t}\Vert ^2] }{T}\nn\\
    &\quad+2(\alpha+K\alpha^2)\frac{k_G}{M}+2(\alpha+K\alpha^2)\nn\\
    &\quad\cdot\frac{4R^2L^2+\left( 1+\frac{1}{\lambda}+\frac{2\gamma(1+k_1R\vert  \mca\vert  )}{\lambda}\right)^2(r_{\max}+\gamma R+R)^2}{T},\nn
\end{align}\end{small}
which implies that \begin{small}
\begin{align}\label{eq:nestedmainresult}
    &\frac{\sum_{t=0}^{T-1}\mathbb{E} [\Vert  \nabla J\left(\theta_t\right)\Vert ^2]}{T}\nn\\
    &\leq \frac{4\left(J\left(\theta_0\right)-J^*\right)}{\left(\alpha-K\alpha^2\right)T}+\frac{8L^2\left(\alpha+K\alpha^2\right)}{\alpha-K\alpha^2}\bigg(4R^2e^{\left(4\beta^2-\beta\lambda\right)(T_c-1)}\nn\\
    &\quad+\left(4\beta^2+\frac{2\beta}{\lambda}\right)\left(\frac{1}{\beta\lambda-4\beta^2} \right)\frac{k_l+k_h}{B}\bigg)\nn\\
    &\quad+\frac{8\left(\alpha+K\alpha^2\right)}{\alpha-K\alpha^2}\frac{k_G}{M}+\frac{8(\alpha+K\alpha^2)}{\alpha-K\alpha^2}\nn\\
    &\quad\cdot\frac{4R^2L^2+\left( 1+\frac{1}{\lambda}+\frac{2\gamma(1+k_1R\vert  \mca\vert  )}{\lambda}\right)^2(r_{\max}+\gamma R+R)^2}{T}\nn\\
    &=\mathcal{O}\left( \frac{1}{T}+\frac{1}{M} +\frac{1}{B}+e^{-T_c}\right),
\end{align}\end{small}
where $L=\gamma (1+\vert  \mca\vert  Rk_1)$.
Now let $T, M, B=\mathcal{O}\left(\frac{1}{\epsilon}\right)$ and $T_c=\mathcal{O}(\log (\epsilon^{-1}))$, we have 
$
    \mathbb{E}[\Vert  \nabla J(\theta_W)\Vert ^2] \leq \epsilon,
$
with the sample complexity
$
    \left(M+T_cB\right)T=\mathcal{O}\left({\epsilon^{-2}}{\log {\epsilon}^{-1}}\right).
$

\subsection{Proof of Lemma \ref{prop:1}}\label{section:nestproof}
Define $z_{t,t_c}=\omega_{t,t_c}-\omega^*(\theta_t)$. Then by the update of $\omega_{t,t_c}$, we have that for any $t\geq 0$, 
\begin{align}
    &z_{t,t_c+1}
    =z_{t,t_c}+\frac{\beta}{B}\sum^B_{i=1} \big(\delta_{(BT_c+M)t+Bt_c+i}(\theta_t)\nn\\
    &\quad-\phi_{(BT_c+M)t+Bt_c+i-1}^\top \omega_{t,t_c}\big)\phi_{(BT_c+M)t+Bt_c+i-1}\nn\\
    &\triangleq z_{t,t_c}+\frac{\beta}{B}\sum^B_{i=1}  l_{t,t_c,i}(\theta_t)-\frac{\beta}{B} \sum^B_{i=1}  h_{t,t_c,i}(z_{t,t_c}),
\end{align}
where $l_{t,t_c,i}(\theta_t)=(\delta_{(BT_c+M)t+Bt_c+i}(\theta_t)-\phi^\top_{(BT_c+M)t+Bt_c+i-1}\omega^*(\theta_t))\phi_{(BT_c+M)t+Bt_c+i-1}$, and $h_{t,t_c,i}(z_{t,t_c})=\phi_{(BT_c+M)t+Bt_c+i-1}^\top z_{t,t_c}\phi_{(BT_c+M)t+Bt_c+i-1}$. 
We also define the expectation of the above two functions under the stationary distribution for any fixed $\theta$ and $z$: $\Bar{l}(\theta)=\mathbb{E}_{\mu}[l_{t,t_c,i}(\theta)]=0$ and $\Bar{h}(z)=\mathbb{E}_{\mu}[h_{t,t_c,i}(z)]=Cz$.
We then have that \begin{small}
\begin{align}\label{eq:nestz}
    &\Vert z_{t,t_c+1}\Vert ^2
    \leq \Vert  z_{t,t_c}\Vert ^2+\frac{2\beta^2}{B^2}\left\Vert {\sum^B_{i=1} l_{t,t_c,i}\left(\theta_t\right)} \right\Vert ^2\nn\\
    &\quad+2\frac{\beta^2}{B^2}\left\Vert {\sum^B_{i=1} h_{t,t_c,i}\left(z_{t,t_c}\right)}  \right\Vert ^2+2\frac{\beta}{B}\left\langle z_{t,t_c},\sum^B_{i=1} l_{t,t_c,i}\left(\theta_t\right)\right\rangle\nn\\
    &\quad-2\frac{\beta}{B}\left\langle z_{t,t_c}, \sum^B_{i=1} h_{t,t_c,i}\left(z_{t,t_c}\right)\right\rangle\nn\\
    &\overset{(a)}{\leq} 
     \left(1+4\beta^2-\beta\lambda\right)\Vert z_{t,t_c}\Vert ^2+\left(2\beta^2+\frac{2\beta}{\lambda}\right)\left\Vert \frac{\sum^B_{i=1} l_{t,t_c,i}\left(\theta_t\right)}{B} \right\Vert ^2\nn\\
    &\quad+\left(4\beta^2+\frac{2\beta}{\lambda}\right)\left\Vert  \Bar{h}\left(z_{t,t_c}\right)-\frac{\sum^B_{i=1} h_{t,t_c,i}\left(z_{t,t_c}\right)}{B}\right\Vert ^2,
\end{align}\end{small}
where $(a)$ is from $\langle z, \Bar{h}(z)\rangle=z^\top C z\geq \lambda \Vert z\Vert ^2$, $\Vert \Bar{h}(z)\Vert ^2=z^\top C^\top C z \leq \Vert z\Vert ^2$ for any $z\in \mathbb{R}^N$,   and $\langle x, y\rangle \leq \frac{\lambda}{4}\Vert x\Vert ^2+ \frac{1}{\lambda}\Vert y\Vert ^2$ for any $x,y \in \mathbb{R}^N$. Recall that $\mathcal{F}_t$ is the $\sigma$-field generated by the randomness until $\theta_t$ and $\omega_t$, hence taking expectation conditioned on $\mathcal{F}_t$ on both sides implies that 
\begin{small}
\begin{align}\label{eq:nestmed}
    &\mathbb{E}[\Vert z_{t,t_c+1}\Vert ^2\vert  \mathcal{F}_t]\nn\\
    &\leq \left(1+4\beta^2-\beta\lambda\right)\mathbb{E}[\Vert z_{t,t_c}\Vert ^2\vert  \mathcal{F}_t]+\left(\frac{2\beta+4\beta^2\lambda}{\lambda B^2}\right)\nn\\
    &\quad\cdot\mathbb{E}\bigg[\bigg\Vert  B\Bar{h}\left(z_{t,t_c}\right)-\sum^B_{i=1} h_{t,t_c,i}\left(z_{t,t_c}\right)\bigg\Vert ^2\bigg\vert  \mathcal{F}_t\bigg]\nn\\
    &\quad+\left(2\beta^2+\frac{2\beta}{\lambda}\right)\mathbb{E}\bigg[\bigg\Vert \frac{\sum^B_{i=1} l_{t,t_c,i}\left(\theta_t\right)}{B} \bigg\Vert ^2\bigg\vert  \mathcal{F}_t\bigg].
\end{align}
\end{small}

From Lemma \ref{nestedmarkov}, it follows that
$
     \mathbb{E}[\Vert z_{t,t_c+1}\Vert ^2\vert  \mathcal{F}_t] \leq (1+4\beta^2-\beta\lambda)\mathbb{E}[\Vert z_{t,t_c} \Vert ^2\vert  \mathcal{F}_t]+\left(4\beta^2+\frac{2\beta}{\lambda}\right)\frac{k_l+k_h}{B}.
$
Choose $\beta<\frac{\lambda}{4}$ and recursively apply the inequality, it follows that \begin{small}
\begin{align}
    &\mathbb{E}[\Vert z_{t+1}\Vert ^2]=\mathbb{E}[\mathbb{E}[\Vert z_{t+1}\Vert ^2\vert  \mathcal{F}_t]]\nn\\
    &\leq 4R^2e^{\left(4\beta^2-\beta\lambda\right)(T_c-1)}+\left(4\beta^2+\frac{2\beta}{\lambda}\right)\left(\frac{1}{\beta\lambda-4\beta^2} \right)\frac{k_l+k_h}{B},\nn
\end{align}\end{small}
which is from $1-x\leq e^{-x}$ for any $x>0$ and $\Vert z_{t,0} \Vert ^2\leq 4R^2$. 
Thus, let $T_c=\mathcal{O}\left(\log \frac{1}{\epsilon}\right),B=\mathcal{O}\left(\frac{1}{\epsilon}\right)$, then $\mathbb{E}[\Vert z_{t}\Vert ^2]\leq \mathcal{O}\left(\epsilon\right)$.
This completes the proof of \eqref{eq:102}. 


\subsection{Lemma \ref{nestedmarkov} and Its Proof}\label{sec:nested_a}
We now present bounds on the ``variance terms" in \eqref{eq:nestmed}.

\begin{Lemma}\label{nestedmarkov}
Consider the Markovian setting, then
\begin{small}
\begin{align}
     &\mathbb{E}\left[\left\Vert \frac{\sum^B_{i=1}  l_{t,t_c,i}(\theta_t)}{B} \right\Vert ^2\Bigg\vert  \mathcal{F}_t\right] \leq  \frac{8(1+\lambda)^{2}(1+\rho(m-1))}{B(r_{\max}+R+\gamma R)^{-2}(1-\rho)};\nn
\end{align}
 
\begin{align}
     \mathbb{E}\left[\left\Vert \frac{\sum^B_{i=1}  h_{t,t_c,i}(z_{t,t_c})}{B} -\Bar{h}(z_{t,t_c}) \right\Vert ^2\bigg\vert  \mathcal{F}_t\right] \leq  \frac{32R^2(1+\rho(m-1))}{B(1-\rho)};\nn
\end{align}\end{small}



\begin{align}
    &\mathbb{E}\left[\left\Vert   {2\hat{G}_t(\theta_t,\omega^*(\theta_t))}+\nabla J(\theta_t)\right\Vert ^2\bigg\vert  \mathcal{F}_t\right]\nn\\
    &\leq  \frac{32\left(1+\lambda+2\gamma(1+Rk_1\vert  \mca\vert  )\right)^2(1+\rho(m-1))}{(r_{\max}+\gamma R+R)^{-2}M(1-\rho)\lambda^2}.\nn
\end{align}
\end{Lemma}
\begin{proof}
Note that $\Bar{l}(\theta)=\mathbb{E}_{\mu}[l_{t,t_c,i}(\theta)]=0$, thus
\begin{small}
\begin{align}
    &\frac{1}{B^2}\mathbb{E}\left[\bigg\Vert \sum^B_{i=1}  l_{t,t_c,i}(\theta_t) \bigg\Vert ^2\bigg\vert  \mathcal{F}_t\right]\nn\\
    &=\frac{1}{B^2}\mathbb{E}\left[\bigg\Vert \sum^B_{i=1}  l_{t,t_c,i}(\theta_t)  -\sum^B_{i=1} \Bar{l}(\theta_t)\bigg\Vert ^2\bigg\vert  \mathcal{F}_t\right]\nn\\
    &=\frac{1}{B^2} \sum^B_{i=1} \mathbb{E}\left[\Vert l_{t,t_c,i}(\theta_t)-\Bar{l}(\theta_t) \Vert ^2\vert  \mathcal{F}_t\right]\nn\\
    &\quad+\frac{1}{B^2} \sum^B_{i\neq j} \mathbb{E}\left[\langle l_{t,t_c,i}(\theta_t)-\Bar{l}(\theta_t),l_{t,t_c,j}(\theta_t)-\Bar{l}(\theta_t) \rangle\vert  \mathcal{F}_t\right]\nn\\
    &\leq\frac{4(1+\lambda)^2(r_{\max}+R+\gamma R)^2}{B}\nn\\
    &\quad+\frac{2}{B^2} \sum^B_{i> j} \mathbb{E}\left[\langle l_{t,t_c,i}(\theta_t)-\Bar{l}(\theta_t),l_{t,t_c,j}(\theta_t)-\Bar{l}(\theta_t) \rangle\vert  \mathcal{F}_t\right],
\end{align}\end{small}
which is due to the fact that  $\vert  l_{s,a,s'}(\theta)\vert  \leq (1+\lambda)(r_{\max}+R+\gamma R)$ for any $(s,a,s')$ and $\Vert \theta\Vert \leq R$. 

For the second part, we first consider the case $i>j$. Let $X_j$ be the $(BT_ct+Mt+Bt_c+j)$-th sample and $X_i$ be the $(BT_ct+Mt+Bt_c+i)$-th sample, and we denote the $\sigma-$field generated by all the randomness until $X_j$ by $\mathcal{F}_{t,t_c,j}$, then
\begin{small}
\begin{align}
    &\mathbb{E}\left[\langle l_{t,t_c,i}(\theta_t)-\Bar{l}(\theta_t),l_{t,t_c,j}(\theta_t)-\Bar{l}(\theta_t) \rangle\vert  \mathcal{F}_t\right]\nn\\
    &=\mathbb{E}\left[\langle \mathbb{E}\left[l_{t,t_c,i}(\theta_t)-\Bar{l}(\theta_t)\vert  \mathcal{F}_{t,t_c,j}\right], l_{t,t_c,j}(\theta_t)-\Bar{l}(\theta_t)\rangle\vert  \mathcal{F}_t\right]\nn\\
    &\leq \mathbb{E}\left[ \left\Vert \mathbb{E}\left[l_{t,t_c,i}(\theta_t)-\Bar{l}(\theta_t)\vert  \mathcal{F}_{t,t_c,j}\right] \right\Vert  \left\Vert l_{t,t_c,j}(\theta_t)-\Bar{l}(\theta_t) \right\Vert \vert  \mathcal{F}_t\right]\nn\\
    &\leq 2(1+\lambda)(r_{\max}+R+\gamma R)\mathbb{E}\left[ \left\Vert \mathbb{E}\left[l_{t,t_c,i}(\theta_t)-\Bar{l}(\theta_t)\vert  \mathcal{F}_{t,t_c,j}\right] \right\Vert \vert  \mathcal{F}_t\right]\nn\\
    &=2(1+\lambda)(r_{\max}+R+\gamma R) \nn\\
    &\quad\cdot \left\Vert   \int_{X_i} l_{X_i}(\theta_t) (dX_i\vert  X_j) -\int_{X_i} l_{X_i}(\theta_t) \mu(dX_i) \right\Vert \nn\\
    &\leq 2(1+\lambda)(r_{\max}+R+\gamma R)\left\Vert  \int_{X_i} l_{X_i}(\theta_t) ((dX_i\vert  X_j)-\mu(dX_i))\right\Vert  \nn\\
    &\leq 2(1+\lambda)^2(r_{\max}+R+\gamma R)^2 \left\vert  \int_{X_i}(dX_i\vert  X_j)-\mu(dX_i) \right\vert  \nn\\
    &\leq 4c_l^2m\rho^{i-j},
\end{align}\end{small}
where the last inequality is from the geometric uniform ergodicity of the MDP.
Thus we have that\begin{small}
\begin{align}
     &\frac{1}{B^2}\mathbb{E}\left[\bigg\Vert \sum^B_{i=1}  l_{t,t_c,i}(\theta_t)\bigg\Vert ^2\vert  \mathcal{F}_t\right]\leq \frac{8(1+\lambda)^2(1+\rho(m-1))}{((r_{\max}+R+\gamma R)^{-2}B(1-\rho)}.\nn
\end{align}\end{small}
Similarly we can show the other two inequalities.
\end{proof}

\end{appendices}

\end{document}